\documentclass{article} % For LaTeX2e
\usepackage[dvipsnames,svgnames]{xcolor}
\usepackage{arxiv,times,natbib}

\usepackage{amsmath,amsfonts,bm}

\def\eqref#1{equation~\ref{#1}}
\def\1{\bm{1}}

\def\va{{\bm{a}}}

\def\ve{{\bm{e}}}

\def\vk{{\bm{k}}}

\def\vq{{\bm{q}}}

\def\vv{{\bm{v}}}

\def\vx{{\bm{x}}}
\def\vy{{\bm{y}}}
\def\vz{{\bm{z}}}

\def\mA{{\bm{A}}}

\def\mE{{\bm{E}}}
\def\mF{{\bm{F}}}

\def\mH{{\bm{H}}}

\def\mS{{\bm{S}}}

\def\mV{{\bm{V}}}

\def\mX{{\bm{X}}}
\def\mY{{\bm{Y}}}
\def\mZ{{\bm{Z}}}

\DeclareMathAlphabet{\mathsfit}{\encodingdefault}{\sfdefault}{m}{sl}
\SetMathAlphabet{\mathsfit}{bold}{\encodingdefault}{\sfdefault}{bx}{n}

\def\sZ{{\mathbb{Z}}}

\newcommand{\E}{\mathbb{E}}

\DeclareMathOperator*{\argmax}{arg\,max}

\newcommand{\cossim}[2]{\cos\left\langle {#1},{#2} \right\rangle}
\newcommand{\llnorm}[1]{\left\lVert{#1}\right\rVert_2}
	
\newcommand{\indep}{\perp \!\!\! \perp}

\newcommand{\meanstd}[2]{{#1}\tiny{$\pm${#2}}}

\usepackage{hyperref}
\hypersetup{
    colorlinks=true,
    citecolor=DarkBlue
}
\usepackage{url}
\usepackage{booktabs}
\usepackage{graphicx}
\usepackage{algorithm} 
\usepackage{algpseudocode}
\usepackage{varwidth}% http://ctan.org/pkg/varwidth
\usepackage{multicol}
\usepackage{multirow}
\usepackage{caption}
\usepackage{subcaption}
\usepackage{wrapfig}
\usepackage{bbm}
\usepackage{amsmath,amssymb}

\usepackage{amsthm}
\newtheorem{theorem}{Theorem}[section]

\newtheorem{lemma}[theorem]{Lemma}
\newtheorem{definition}[theorem]{Definition}
\newtheorem{assumption}[theorem]{Assumption}

\def\modelfullname{Counterfactual Self-Supervised Transformer}
\def\modelshortname{COSTAR}

\title{\modelshortname: Improved Temporal Counterfactual Estimation with Self-Supervised Learning}

\author{Chuizheng Meng \thanks{University of Southern California}\hspace{0.5em}\thanks{Google Cloud AI Research}\\
\texttt{chuizhem@usc.edu} \\
\And
Yihe Dong \footnotemark[2] \\
\texttt{yihed@google.com} \\
\And
Sercan \"{O}. Ar{\i}k \footnotemark[2] \\
\texttt{soarik@google.com} \\
\AND
Yan Liu \footnotemark[1]\hspace{0.5em}\footnotemark[2] \\
\texttt{yanliu.cs@usc.edu}
\And
Tomas Pfister \footnotemark[2]\\
\texttt{tpfister@google.com} \\
}

\makeatletter
\@newctr{footnote}[page]
\makeatother
\newcommand{\rebuttal}[1]{{#1}}

\newcommand\blfootnote[1]{
    \begingroup
    \renewcommand\thefootnote{}\footnote{#1}
    \addtocounter{footnote}{-1}
    \endgroup
}

\begin{document}

\maketitle

\begin{abstract}
Estimation of temporal counterfactual outcomes from observed history is crucial for decision-making in many domains such as healthcare and e-commerce, particularly when randomized controlled trials (RCTs) suffer from high cost or impracticality. For real-world datasets, modeling time-dependent confounders is challenging due to complex dynamics, long-range dependencies and both past treatments and covariates affecting the future outcomes. In this paper, we introduce \modelfullname~(\modelshortname), a novel approach that integrates self-supervised learning for improved historical representations. We propose a component-wise contrastive loss tailored for temporal treatment outcome observations and explain its effectiveness from the view of unsupervised domain adaptation. \modelshortname~yields superior performance in estimation accuracy and generalization to out-of-distribution data compared to existing models, as validated by empirical results on both synthetic and real-world datasets.
\blfootnote{Code available at \url{https://github.com/google-research/google-research/tree/master/COSTAR}}
\end{abstract}

\section{Introduction}
% \sercan{
% \begin{itemize}
%     \item Description of cold-start forecasting, and the explanation of challenges, including being prone to overfitting and poor generalization with distribution shifts
%     \item Why causal discovery and accurate counterfactual modeling are important to improve cold-start forecasting
%     \item Prior work on counterfactual modeling and their shortcomings
%     \item Contributions
% \end{itemize}
% }

Accurate estimation of treatment outcomes over time conditioning on the observed history is a fundamental problem in causal analysis and decision making in various applications \citep{mahar2021scoping,ye2023web,wang2023multi}. For example, in medical domains, doctors are interested in knowing how a patient reacts to a treatment or multi-step treatments; in e-commerce, retailers are concerned about how future sales change if adjusting the price of an item. While randomized controlled trials (RCTs) are the gold standard for treatment outcome estimation, most often than not such trials are either too costly or even impractical to conduct. Therefore, utilizing available observed data (such as electronic health records (EHRs) and historical sales) for accurate treatment outcome estimation, has drawn increasing interest in the community.

Compared to the well-studied i.i.d cases, treatment outcome estimation from time series observations not only finds more applications in the real world but also pose significant more challenges, due to the complex dynamics and the long-range dependencies in time series.
%\textcolor{red}{I don't think we focus on addressing confounders and long-range dependencies. Maybe we should move on to quickly to the challenges we address.}
Existing works along this endeavors explore various architectures with improved capacity and training strategies to alleviate time-dependent confounding\footnote{We leave a more detailed review of related work in Sec.~\ref{sec:related-work} of the appendix.}. Recurrent marginal structural networks (RMSNs)~\citep{lim2018forecasting}, counterfactual recurrent networks (CRN)~\citep{Bica2020Estimating}, and G-Net~\citep{li2021g} utilize architectures based on recurrent neural networks. To mitigate time-dependent confounding, they train proposed models with inverse probability of treatment weighting (IPTW), treatment invariant representation through gradient reversal, and G-computation respectively, in addition to the factual estimation loss on observed data. Causal Transformer (CT)~\citep{melnychuk2022causal} further improves capturing long-range dependencies in the observational data with a tailored transformer-based architecture and overcomes the temporal confounding with balanced representations trained through counterfactual domain confusion loss. 

\begin{wrapfigure}{r}{0.5\textwidth}
% \vspace{-2ex}
    \centering
    \includegraphics[width=\linewidth]{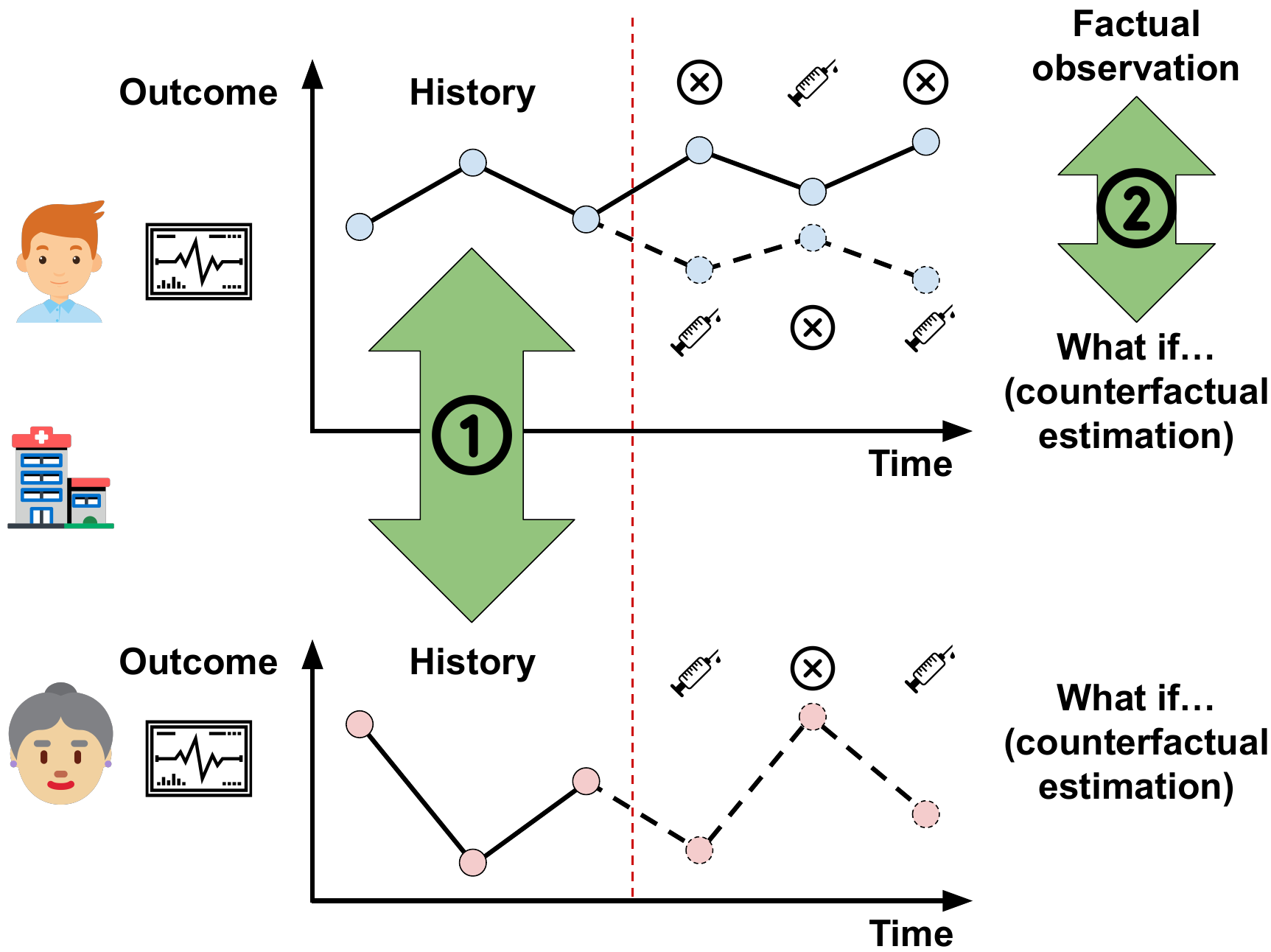}
    \caption{We illustrate the problem of treatment outcome estimation over time with an example in healthcare. 
    % With observed history available only on a subpopulation, our goal is to learn an outcome estimator that generalizes to unseen subpopulations and estimates counterfactual outcomes under treatments unbiased towards history. 
    We propose \modelshortname~as a temporal counterfactual estimator enhanced with self-supervised learning, inducing transferrability to both \textcircled{1}  cold-start cases from unseen subpopulations and \textcircled{2} counterfactual outcome estimation.}
    \label{fig:ssl-cold-start}
% \vspace{-4ex}
\end{wrapfigure}

While existing methods achieve performance gain in empirical evaluation, they rely on the fully supervised loss of future outcomes to learn representations of history and thus suffer from its limitations.
% First, the learned representations focus on predicting future outcomes in a biased way \textcolor{red}{can we explain why it is biased?} and cannot effectively capture discriminative information about past and future covariates. 
% Second, 
In many practical applications, we are confronted with the cold case challenge, where no or limited observations of testing time series are accessible. 
% For example, after training the model with historical sales (seen as outcomes) and pricing (seen as treatments) sequences of products in the food category, we are asked to estimate the sales of new items in the household category given their pricing, with no or very limited observations of household items collected beforehand. 
Figure~\ref{fig:ssl-cold-start} shows an example in healthcare: after training a vital sign estimator with historical health records (seen as outcomes) and drug usage (seen as treatments) sequences of patients in the youth age group, the model is asked to estimate the potential vital signs in the elderly age group after applying a treatment plan, with no or very limited observations of elderly people collected beforehand. 
Existing methods based on supervised learning 
% not only fail to effectively capture the dependencies of outcomes on observed feature distributions but also 
have difficulty generalizing to different domains and handling cold cases in test time. 

%Another difficulty is the existence of time-dependent confounders - both past treatments and covariates affect current treatments as well as outcomes. Estimators directly minimizing the empirical risk on observed data might thus contain biases due to its incapability of isolating the effect of past and current treatments.

% add discussions on cold cases

In this work, we propose a paradigm shift from supervised learning to self-supervised training for temporal treatment outcome estimation. Our proposed model, \textbf{\modelfullname~(\modelshortname)}, addresses the aforementioned limitations. To enhance the model capacity, we propose an encoder architecture composed of alternating temporal and feature-wise attention, capturing dependencies among both time steps and features. To learn expressive and transferable representations of the observed history, we refine the contrastive loss in self-supervised learning to a finer-grained level: both the entire history and each of the covariate/treatment/outcome components are contrasted when constructing the loss. Moreover, we view the counterfactual outcome estimation problem from the unsupervised domain adaptation (UDA) perspective and provide the theoretical analysis of the error bound for a counterfactual outcome estimator that gives estimation based on representations from self-supervised learning.

Our main contributions are summarized as follows:
\begin{enumerate}
    \item We adapt self-supervised learning (SSL) together with component-wise contrastive losses tailored for temporal observations to learn more expressive representations of the history in temporal counterfactual outcome estimation.
    % \item We propose a new encoder architecture combining both temporal attention and feature-wise attention for better modeling the complex temporal dependencies and feature interactions in observations.
    \item We explain the boost from self-supervised learning on the counterfactual outcome estimation problem with the view of unsupervised domain adaptation (UDA) perspective and provide the theoretical analysis of the error bound of a counterfactual outcome estimator that predicts with representations from self-supervised learning.
    % \item We propose a causal structure based data augmentation method to construct positive pairs in contrastive learning and present theoretical analysis showing that the counterfactual outcome estimation with augmented data is equivalent to unsupervised domain adaptation.
    % and thus benefits from self-supervised learning.
    % \vspace{-10ex}
    \item Empirical results show that our proposed framework outperforms existing baselines across both synthetic and real-world datasets in both estimation accuracy and generalization. In addition, we demonstrate that the learned representations are balanced towards treatments and thus address the temporal confounding issue.% in a superior way.
\end{enumerate}

% \todo{
% \begin{itemize}
%     % \item Drawbacks: limited by fully supervised learning (only keeps information about factual outcomes; less generalizable)
%     \item Our method: (1) SSL learns more expressive representations, which address the temporal confounding issue - balanced representations towards treatments and generalize to data with different distributions. (2) Propose new encoder architecture combining both temporal attention and feature-wise attention. (3) Demonstrate state-of-the-art performance on both normal cases and cold-start examples across synthetic and real-world datasets.
% \end{itemize}
% }
\section{Related Work}
\label{sec:related-work}
\paragraph{Counterfactual treatment outcome estimation over time.}
Early works in counterfactual treatment outcome estimation were first developed for epidemiology and can be considered under 3 major groups: G-computation, marginal structural models (MSMs), and structural nested models~\citep{robins1986new,robins1994correcting,robins2000marginal,robins2008estimation}. One major shortcoming of these is that they are built on linear models and suffer from the limited model capacity when facing time series data with complex temporal dependencies. Follow-up works address the limitation in expressiveness with Bayesian non-parametric methods \citep{xu2016bayesian,soleimani2017treatment,schulam2017reliable} or more expressive deep neural networks (DNNs) such as recurrent neural networks (RNNs). For example, recurrent marginal structural networks (RMSNs) \citep{lim2018forecasting} replace the linear model in MSM with an RNN-based architecture for forecasting treatment outcomes. G-Net \citep{li2021g} also adopts RNN instead of classical regression models in the g-computation framework.
Inspired by the success of representation learning for domain adaptation and generalization \citep{ganin2016domain,tzeng2015simultaneous}, more recent works explore learning representations that are both predictive for outcome estimation and balanced regardless of the treatment bias in training data. Counterfactual recurrent network (CRN) \citep{Bica2020Estimating} trains an RNN-based model with both the factual outcome regression loss and the gradient reversal \citep{ganin2016domain} w.r.t. the treatment prediction loss. The former loss encourages the learned representations to be predictive of outcomes while the latter encourages the representations to be homogeneous given different treatments. The joint training target leads to informative and balanced representations. With similar motivations, \citep{melnychuk2022causal} replaces the RNN-based architecture with a Transformer-based \citep{vaswani2017attention} one along with the domain confusion loss \citep{tzeng2015simultaneous} to learn treatment-agnostic representations. Given the flexibility of the choice of model architectures, recent works extend temporal counterfactual outcome estimation to irregular time series~\citep{seedat2022continuous,cao2023estimating}, temporal point process \citep{zhang2022counterfactual}, and graph-structured spatiotemporal data~\citep{jiang2023cf} with the help of \citep{kidger2020neural} and \citep{huang2020learning}. While existing works claim that both predictive and balanced representations are critical in accurate counterfactual outcome estimation, we empirically find that the impact of representation balancing is inconsistent and marginal. In contrast, improving the expressiveness of representations brings more robust improvements.

\paragraph{Self-supervised learning of time series.}
Being widely studied first for computer vision tasks \citep{he2020momentum,chen2020simple,grill2020bootstrap,chen2021mocov3}, self-supervised learning achieves strong performance with the advantage of not relying on labeled data. Recent works \citep{yue2022ts2vec,tonekaboni2021unsupervised,woo2022cost,zhang2022self} further generalize and adapt self-supervised learning methods for time series, including classification, forecasting, and anomaly detection tasks. However, existing works of counterfactual outcome estimation over time have neglected self-supervised learning of time series as an effective way of learning informative representations. Meanwhile, existing models for self-supervised learning of time series are not tailored for counterfactual outcome estimation. Hence we propose \modelshortname~to mitigate the gap.
\section{Problem Formulation}
\label{sec:formulation}
Our task is estimating the outcomes of subjects with observed history after being applied a sequence of treatments from observational data \citep{lim2018forecasting,Bica2020Estimating,melnychuk2022causal}. We represent the available observed dataset as $\{\{\vx^{(i)}_t, \va^{(i)}_t,\vy^{(i)}_t\}_{t=1}^{T^{(i)}}, \vv^{(i)}\}_{i=1}^N$ of $N$ independently sampled subjects, where $T^{(i)}\in\mathbb{N}^+$ denotes the length of the observed history of subject $i$, and $\vx^{(i)}_t\in\mathbb{R}^{d_X}$, $\va^{(i)}_t\in\mathbb{R}^{d_A}$, and $\vy^{(i)}_t\in\mathbb{R}^{d_Y}$ stand for the observed vector of covariates, treatments, and outcomes respectively, at time $t$ of subject $i$. $\vv^{(i)}\in\mathbb{R}^{d_V}$ contain all static features of subject $i$. We omit the subject index $i$ in the following text for notational simplicity.

Following the potential outcomes~\citep{splawa1990application,rubin1978bayesian} framework extended to time-varying treatments and outcomes~\citep{robins2008estimation}, our target is to estimate $\E(\vy_{t+\tau}[\bar{\va}_{t:t+\tau-1}] \vert \bar{\mH}_t)$ for $\tau \geq 1$, where $\bar{\mH}_t = (\bar{\mX}_t, \bar{\mA}_{t-1}, \bar{\mY}_t, \mV)$ is the observed history. $\bar{\mX}_t = (\vx_1, \vx_2, \dots, \vx_t)$, $\bar{\mA}_{t-1} = (\va_1, \va_2, \dots, \va_{t-1})$, $\bar{\mY}_t = (\vy_1, \vy_2, \dots, \vy_t)$, $\mV = \vv$. $\bar{\va}_{t:t+\tau-1} = (\va_t, \va_{t+1}, \dots, \va_{t+\tau - 1})$ is the sequence of the applied treatments in the future $\tau$ discrete time steps.
In factual data, $\bar{\mH}_t$ and $\va_t$ are correlated, leading to the treatment bias in counterfactual outcome estimation. In addition, the distribution of $\bar{\mH}_t$ can also vary between training and test data: $P_{\mathcal{D}_{tr}}(h) \neq P_{\mathcal{D}}(h)$, causing the feature distribution shifts. Following the tradition in domain adaptation, we name $P_{\mathcal{D}_{tr}}(h)$ and $P_{\mathcal{D}}(h)$ the source/target domains respectively.  Table~\ref{tab:data-stat} describes feature distribution shifts in our datasets.

To ensure the identifiability of treatment effects from observational data, we take the standard assumptions used in existing works~\citep{Bica2020Estimating,melnychuk2022causal}:
% consistency, positivity and sequential strong ignorability.
(1) consistency, (2) positivity and (3) sequential strong ignorability (See Appendix~\ref{sec:identifiability}).

% \begin{assumption}[Consistency]
% \label{asm:consistency}
% The potential outcome of any treatment $\va_t$ is always the same as the factual outcome when a subject is given the treatment $\va_t$: $\vy_{t+1}[\va_t] = \vy_{t+1}$.
% \end{assumption}

% \begin{assumption}[Positivity]
% If $P(\bar{\mA}_{t-1} = \bar{\va}_{t-1}, \bar{\mX}_t = \bar{\vx}_t) \neq 0$, then $P(\mA_t = \va_t \vert \bar{\mA}_{t-1} = \bar{\va}_{t-1}, \bar{\mX}_t = \bar{\vx}_t) > 0$ for any $\bar{\va}_t$.
% \end{assumption}

% \begin{assumption}[Sequential strong ignorability]
% $\mY_{t+1}[\va_t] \indep \mA_t \vert \bar{\mA}_{t-1}, \bar{\mX}_t, \forall \va_t, t$.
% \end{assumption}

\section{\modelfullname}
% \todo{
% \begin{enumerate}
%     % \item \textcolor{green}{\text{[Done]} Adjusting the encoder architecture to make it more different from the layers in CausalTransformer: apply full attention among features in addition to the current causal temporal attention.}
%     % \item \textcolor{green}{\text[Done] Use different weights to penalize prediction errors in different steps.}
%     \item A theorem showing the generalization bound of learned history representations based on works of contrastive learning and unsupervised domain generalization \citep{cai2021theory,shen2022connect}.
% \end{enumerate}
% }

We illustrate the detailed design of our proposed \modelfullname~(\modelshortname). Our main goal is to learn representations of observed history sequences that are informative for counterfactual treatment outcome estimation, which we achieve by tailoring both the representation encoder architecture and the self-supervised training loss. On top of the representation learning, we also propose a simple yet effective decoder for non-autoregressive outcome prediction and demonstrate improvements in both the accuracy and the speed of multi-step estimation. Fig.~\ref{fig:overview} overviews the proposed framework.

\begin{figure}[t]
    \centering
    \begin{subfigure}[b]{\linewidth}
        \includegraphics[width=\linewidth]{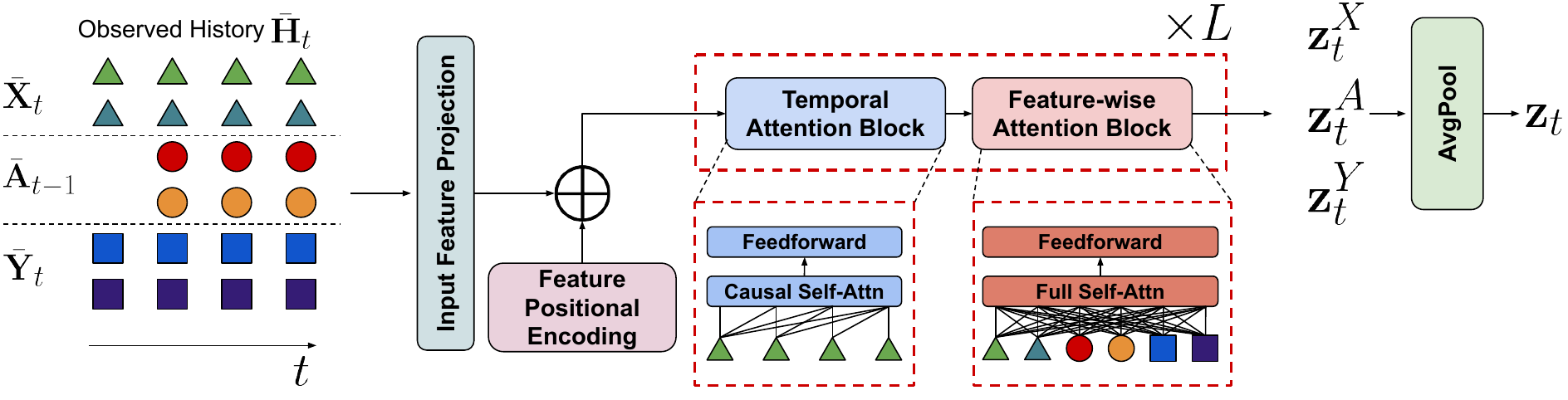}
        \caption{Encoder architecture.}
        \label{fig:encoder}
    \end{subfigure}
    \begin{subfigure}[b]{0.45\linewidth}
        \includegraphics[width=\linewidth]{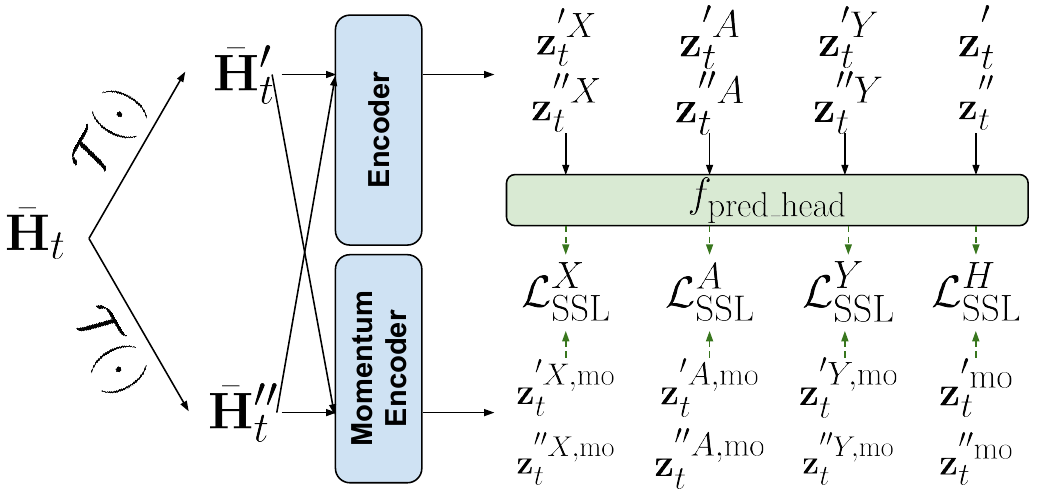}
        \caption{Self-supervised learning of history.}
        \label{fig:ssl-overview}
    \end{subfigure}
    \hspace{0.05\textwidth}
    \begin{subfigure}[b]{0.45\linewidth}
        \includegraphics[width=\linewidth]{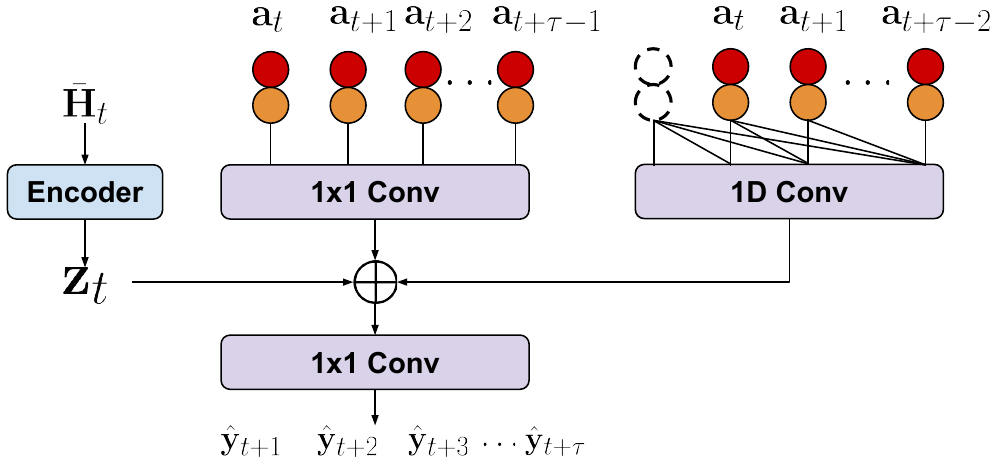}
        \caption{Non-autoregressive outcome predictor.}
        \label{fig:predictor}
    \end{subfigure}
    \caption{Overview of \modelshortname. 
    (a) Encoder architecture.
    % The encoder is composed of $L$ basic layers, and each layer contains a temporal attention block followed by a feature-wise attention block.
    % Both blocks are modified from the standard building block of the vanilla Transformer but differ in the dimensions that the attention mechanism applies along.
    The Temporal Attention Block applies temporal causal attention along the time dimension in parallel for each feature, while the Feature-wise Attention Block calculates full self-attention along the feature dimension in all time steps.
    % The alternating attention dimension allows the interaction among both time steps and features while keeping the time and memory complexity unchanged from the vanilla transformer.
    % After the last layer, we average the representations of covariates, treatment variables and outcome variables into the representations of historical covariates, treatments and outcomes.
    (b) Self-supervised learning of the history representations. Positive pairs are generated by applying random transformations $\mathcal{T}(\cdot)$ on the same sample. We construct component-wise contrastive losses of historical covariates, treatments and outcomes in addition to the standard contrastive loss of the entire sequence. 
    (c) Non-autoregressive outcome predictor architecture.
    % For each future step $t + \tau$, we sum the representation of the history, the encoding of the treatment $\mathbf{a}_{t+\tau-1}$ that is applied just before $t+\tau$, and the encoding of the remaining treatments $(\mathbf{a}_t,\dots,\mathbf{a}_{t+\tau-2})$ as the input of a multi-layer perceptron network predicting future outcomes.
    }
    \label{fig:overview}
\end{figure}

\subsection{Encoder architecture}
\label{sec:encoder_arch}
% \todo{A transformer based encoder with both (1) causal self-attention along the temporal dimension and (2) full self-attention along the feature dimension. }
% \todo{
% \begin{itemize}
%     \item input feature proj \& feature positional encoding
%     \item temporal attention block
%     \item feature-wise attention block
% \end{itemize}
% }

% Fig.~\ref{fig:encoder} demonstrates the architecture of the proposed encoder. We first embed each scalar value across all time steps and features in the given history $\bar{\mH}_t$ into a $d_{\text{model}}$-dim vector and add the positional encoding along the feature dimension as the input of the $L$ encoder layers. Each layer contains two attention blocks, temporal attention and feature-wise attention, and hence the flow is alternated between them. The propagated embeddings are grouped and aggregated based on Eq.~\ref{eq:rep-agg}.
For a given sequence of the observed history $\bar{\mH}_t\in\mathbb{R}^{t\times d_{\text{input}}}$ concatenated from $(\bar{\mX}_t, \bar{\mA}_{t-1}, \bar{\mY}_t)$ (we omit static variables $\mV$ here for simplicity and leave the processing details in Sec.~\ref{sec:encoder_arch}; $d_\text{input} = d_X + d_A + d_Y$), the encoder in Fig.~\ref{fig:encoder} maps the entire history to representations $\{\vz^{i}_t \in \mathbb{R}^{d_{\text{model}}}\}_{i=1}^{d_{\text{input}}}$ for each feature $f_i$.
% through attention blocks alternating between temporal attention and feature-wise attention. 
Then, we employ average pooling for feature-wise representations from corresponding features to get the representations of covariate, treatment, and outcome components, and all features for the representation of the entire observed history. Denote the set of covariate, treatment, and outcome variables as $\mathbf{F}_X$, $\mathbf{F}_A$, and $\mathbf{F}_Y$ respectively. We have:
\begin{equation}
\label{eq:rep-agg}
\begin{array}{ll}
    \vz^{X}_t = \text{avg}(\{\vz^i_t\}_{f_i\in\mathbf{F}_X}),&\vz^{A}_t = \text{avg}(\{\vz^i_t\}_{f_i\in\mathbf{F}_A}),\\
    \vz^{Y}_t = \text{avg}(\{\vz^i_t\}_{f_i\in\mathbf{F}_Y}),&\vz_t = \text{avg}(\{\vz^i_t\}_{i=1}^{d_{\text{input}}}).
\end{array}
\end{equation}

We describe the detailed design of the encoder architecture in Appendix~\ref{sec:detail-encoder}.

\subsection{Self-supervised representation learning of the observed history}
% \todo{InfoNCE lossses on (1) the representations of the entire history; (2) representations of covariates, previous treatments, previous outcomes respectively.}
% \todo{
% \begin{itemize}
%     \item MoCov3 based self-supervised training
%     \item Perturbation for time series.
%     \item Forms of contrastive losses.
% \end{itemize}
% }
We employ pretraining for the encoder in a self-supervised way with the contrastive learning objectives $\mathcal{L}^X_{\text{SSL}}$, $\mathcal{L}^A_{\text{SSL}}$, $\mathcal{L}^Y_{\text{SSL}}$, $\mathcal{L}^H_{\text{SSL}}$ for the component representations $\vz^X_t, \vz^A_t, \vz^Y_t$ and the overall representation $\vz_t$ respectively. The overall self-supervised learning loss is given as:
\begin{equation}
    \mathcal{L}_{\text{SSL}} = \mathcal{L}^H_{\text{SSL}} + (\mathcal{L}^X_{\text{SSL}} + \mathcal{L}^A_{\text{SSL}} + \mathcal{L}^Y_{\text{SSL}}) / 3.
\end{equation}

\paragraph{Self-supervised training.} We train our proposed encoder to learn representations of the history with a self-supervised learning framework modified based on MoCo v3~\citep{chen2021mocov3} that achieves the state-of-the-art performance in self-supervised vision transformer training. Following MoCo v3, we set up our proposed encoder $f_\text{enc}$ as combination of a momentum encoder with the same architecture and initial weights $f^{\text{mo}}_\text{enc}$, and a multi-layer perceptron (MLP) as the prediction head $f_\text{pred\_head}: \mathbb{R}^{d_\text{model}}\rightarrow \mathbb{R}^{d_\text{model}}$. We first apply the augmentations in~\citep{woo2022cost}, which includes scaling, shifting and jittering, on each sample in the input batch $\{\bar{\mH}^{(i)}_t\}_{i=1}^{B}$ ($B$ is the batch size) and generates the positive sample pair $\{\bar{\mH}_t^{'(i)}\}_{i=1}^B$, $\{\bar{\mH}_t^{''(i)}\}_{i=1}^B$. Their representations are encoded as follows:
\begin{equation}
\begin{array}{ll}
    \vz_t^{'X(i)}, \vz_t^{'A(i)}, \vz_t^{'Y(i)}, \vz_t^{'(i)} &= f_\text{enc}(\bar{\mH}^{'(i)}_t),\\
    \vz_t^{''X(i)}, \vz_t^{''A(i)}, \vz_t^{''Y(i)}, \vz_t^{''(i)} &= f_\text{enc}(\bar{\mH}^{''(i)}_t),\\
    \vz_t^{'X, \text{mo}(i)}, \vz_t^{'A, \text{mo}(i)}, \vz_t^{'Y, \text{mo}(i)}, \vz_t^{'\text{mo}(i)} &= f_\text{enc}^{\text{mo}}(\bar{\mH}^{'(i)}_t),\\
    \vz_t^{''X, \text{mo}(i)}, \vz_t^{''A, \text{mo}(i)}, \vz_t^{''Y, \text{mo}(i)}, \vz_t^{''\text{mo}(i)} &= f_\text{enc}^{\text{mo}}(\bar{\mH}^{''(i)}_t).
\end{array}
\end{equation}

The vanilla MoCo v3 method adopts the InfoNCE contrastive loss~\citep{oord2018representation} as the training objective:
\begin{equation}
\label{eq:overall-con-loss}
\begin{aligned}
    \mathcal{L}_{\text{SSL}}^H = &\mathcal{L}_\text{InfoNCE}(\{f_\text{pred\_head}(\vz_t^{'(i)})\}_{i=1}^B, \{\vz_t^{''\text{mo}(i)}\}_{i=1}^B) \\
    &+ \mathcal{L}_\text{InfoNCE}(\{f_\text{pred\_head}(\vz_t^{''(i)})\}_{i=1}^B, \{\vz_t^{'\text{mo}(i)}\}_{i=1}^B),
\end{aligned}
\end{equation}where
\begin{equation}
\begin{aligned}
    &\mathcal{L}_\text{InfoNCE}(\{\vq^{(i)}\}_{i=1}^B, \{\vk^{(i)}\}_{i=1}^B) \\
    = &-\frac{1}{B}\sum_{i=1}^B\log\frac{\exp(\cossim{\vq^{(i)}}{\vk^{(i)}})}{\sum_{j=1}^B\exp(\cossim{\vq^{(i)}}{\vk^{(j)}})}.
\end{aligned}
\end{equation}

% \paragraph{Augmentation of time series data.}
% Since the vanilla MoCo v3 is developed for computer vision tasks, we adapt the input augmentation part for time series input in our task.

\paragraph{Component-wise contrastive loss.}
In addition to the contrastive loss of the overall representations, we enhance the training with contrastive losses on each subset of covariates, treatments and outcomes:
\begin{equation}
\begin{aligned}
    &\mathcal{L}_{\text{SSL}}^{(\cdot)} \\
    = &\mathcal{L}_\text{InfoNCE}(\{f_\text{pred\_head}(\vz_t^{'(\cdot)(i)})\}_{i=1}^B, \{\vz_t^{''(\cdot),\text{mo}(i)}\}_{i=1}^B) \\
    &+\mathcal{L}_\text{InfoNCE}(\{f_\text{pred\_head}(\vz_t^{''(\cdot)(i)})\}_{i=1}^B, \{\vz_t^{'(\cdot),\text{mo}(i)}\}_{i=1}^B),
\end{aligned}
\end{equation}
where $(\cdot)$ is $X, A, Y$.

\subsection{Non-autoregressive outcome predictor}
The architecture of the proposed predictor model is shown in Fig.~\ref{fig:predictor}. At the prediction stage, we first encode the observed history $\bar{\mH}_t$ with the pretrained encoder, the treatment $\va_{t' - 1}$ is modeled right before time $t'=t+1,\dots,t+\tau-1$ with a 1x1 convolution layer, and the remaining treatment sequence $(\va_{t}, \va_{t+1}, \dots, \va_{t'-2})$ with a 1D convolution layer. Then, the concatenated encoding is fed into a multi-layer perceptron (MLP) to predict future outcomes $(\hat{\vy}_{t+1}, \hat{\vy}_{t+2}, \dots, \hat{\vy}_{t+\tau})$. We jointly train the predictor layers and fine tune the pretrained encoder with the $L_2$ loss of factual outcome estimation weighted for each step:
\begin{equation}
\label{eq:suploss}
    \mathcal{L}_{\text{est}} = \sum_{i=1}^{\tau} w_i \left\lVert \hat{\vy}_{t+i} - \vy_{t+i}  \right\rVert_2^2,
\end{equation}
where each $w_i$ is a hyperparameter satisfying $\sum_{i=1}^{\tau}w_i = 1$ and various strategies of setting $w_i$s can be selected via validation errors, which we will discuss in the ablation study (Sec.~\ref{sec:ablation}).

% \subsection{Generalization bound of contrastive learning in counterfactual outcome estimation}
\subsection{Unsupervised domain adaptation view of counterfactual outcome estimation}
For the given observed history $\bar{\mH}_t$, treatment sequence $\bar{\va}_{t:t+\tau-1}$ to apply, and the outcome $\vy_{t+\tau}[\va_{t:t+\tau-1}]( \bar{\mH}_t)$ to estimate, we notice that learning a counterfactual treatment outcome estimator $f_{\va}(\bar{\mH}_t) = \E(\vy_{t+\tau}[\va] \vert \bar{\mH}_t)$ with factual data specifically for a certain treatment sequence $\va$ can be viewed as an unsupervised domain adaptation (\textbf{UDA}) problem with the treatment value being discrete -- any sample in the factual dataset can be categorized into one of the two subsets (1) \textbf{$\mathcal{S}_{\va}=\{(\bar{\mH}_t^{(i)}, \bar{\va}_{t:t+\tau-1}^{(i)}, \vy_{t+\tau}^{(i)})\}_{\bar{\va}^{(i)}_{t:t+\tau-1} = \va}$} and (2) \textbf{$\mathcal{T}_{\bar{\va}}=\{(\bar{\mH}_t^{(i)}, \bar{\va}_{t:t+\tau-1}^{(i)}, \vy_{t+\tau}^{(i)})\}_{\bar{\va}^{(i)}_{t:t+\tau-1} \neq \va}$}. With Assumption \ref{asm:consistency}, we have $\vy_{t+\tau}^{(i)}=\vy_{t+\tau}[\va](\bar{\mH}_t^{(i)})$ in $\mathcal{S}_{\va}$ and thus $\vy_{t+\tau}^{(i)}$ is a label of $f_{\va}(\bar{\mH}_t)$. This does not hold for $\mathcal{T}_{\va}$, where $\bar{\va}^{(i)}_{t:t+\tau-1}\neq\va$. Therefore, $\mathcal{S}_{\va}$ and $\mathcal{T}_{\va}$ correspond to the \textbf{labeled} and the \textbf{unlabeled} dataset in UDA. Considering the existence of treatment bias,  $P_{\mathcal{T}_{\va}}(\bar{\mH}_t) = P(\bar{\mH}_t \vert \bar{\va}_{t:t+\tau-1} \neq \va) \neq P(\bar{\mH}_t \vert \bar{\va}_{t:t+\tau-1} = \va) = P_{\mathcal{S}_{\va}}(\bar{\mH}_t)$, which corresponds to the distribution shift between the labeled source domain and the unlabeled target domain in UDA.
\rebuttal{Notice that the source/target domains here are used for describing the labeled and unlabeled subsets regarding a treatment value to help us analyze the error of counterfactual outcome estimation in UDA framework, and are different from the definitions we use in Section~\ref{sec:formulation}. In the latter case, source/target domains describe the different distributions of $\bar{\mH}_t$ between train/test data.}
As a natural generalization, we analyze the upper bound of contrastive learning for counterfactual outcome estimation based on the transferability analysis of contrastive learning in UDA \citep{haochen2022beyond}:
% \vspace{-4ex}
\begin{theorem}[Upper bound of counterfactual outcome estimator]
\label{theorem:upper-bound-cont}
Suppose that Assumptions~\ref{asm:cross-cluster},~\ref{asm:intra-cluster}, and \ref{asm:relative-expansion} hold for the set of observed history $\mathcal{H}$ and its positive-pair graph $G(\mathcal{H}, w)$, and the representation dimension $k\geq 2m$. Let $r$ be a minimizer of the generalized spectral contrastive loss on factual data and the regression head $f_\va$ be constructed in Alg. \ref{alg:pfa} with labeled data. We have
\begin{equation}
\begin{aligned}
    \mathcal{E}_\mathcal{D}(f_\va) \lesssim &P(\va)\mathcal{E}_{\mathcal{S}_\va}(f_\va) + (1 - P(\va))\\
    &\cdot\left[\epsilon^2 + (4B^2 - \epsilon^2)\frac{r}{\alpha^2\gamma^4}\cdot\exp(-\Omega(\frac{\rho \gamma^2}{\alpha^2}))\right],
\end{aligned}
\end{equation}
\end{theorem}
where $P(\va)$ is the prior probability of the treatment $\va$ to apply. $\mathcal{E}_{\mathcal{S}_\va}(f_\va)$ is the outcome estimation error of $f_\va$ in $\mathcal{S}_{\va}$, which can be optimized with supervised learning. $\alpha$, $r$, $\gamma$, $\rho$ are parameters in Assumptions~\ref{asm:cross-cluster},~\ref{asm:intra-cluster},~\ref{asm:relative-expansion}. $\epsilon$ is the hyperparameter in Alg. \ref{alg:pfa}. $B$ is the upper bound of outcome and predicted outcome values\footnote{Bounded outcome values can be achieved through normalization.}.
When $\gamma \geq \alpha^{1/2}$ and $\rho$ is comparable to $\alpha$, $\rho\gamma^2 \gg \alpha^2$ and lead to a small upper bound.

% \subsection{Unsupervised domain adaptation view of counterfactual outcome estimation}
\textbf{Sketch of the proof}\quad
For the given observed history $\bar{\mH}_t$, treatment sequence $\bar{\va}_{t:t+\tau-1}$ to apply, and the outcome $\vy_{t+\tau}[\va_{t:t+\tau-1}]( \bar{\mH}_t)$ to estimate, we slightly abuse scalar/vector/matrix notations and denote them as $h, a, y[a](h)$ for simplicity. With discrete treatments\footnote{For a sequence of discrete treatments, we can always map it to a single discrete variable with a proper encoding.}, we notice that the counterfactual outcome prediction of each type of treatments from factual data can be viewed as an unsupervised domain adaptation (\textbf{UDA}) problem:

For a treatment type $a$, we aim at finding a function $f_a(h)$ that specifically estimates $\E(y[a](h))$. Any sample $(h_i, a_i, y_i)$ from the observed dataset $\mathcal{D}_{tr}$ can be categorized into one of (1) \textbf{Labeled subset $\mathcal{S}_a=\{(h_i, a_i, y_i)\}_{a_i=a}$} and (2) \textbf{Unlabeled subset $\mathcal{T}_a=\{(h_i, a_i, y_i)\}_{a_i\neq a}$}. According to Assumption \ref{asm:consistency}, for any $(h_i, a_i, y_i)\in\mathcal{S}_a$, we have $y_i=y[a](h_i)$ and thus $y_i$ is a label of $f_a(h)$. In contrast, for any $(h_i, a_i, y_i)\in\mathcal{T}_a$, $y_i=y[a_i](h_i)$ and $a_i\neq a$, resulting in that $y_i$ is no longer a valid label of $f_a(h)$. For simplicity, we omit the treatment symbol $a$ as well as the $y_i$ in $\mathcal{T}_a$: $\mathcal{S}_a=\{(h_i, y_i)\}_{a_i=a}$, $\mathcal{T}_a=\{h_i\}_{a_i\neq a}$.

Considering the existence of treatment bias, there exists at least a $a'\neq a$ satisfying $P(h\vert a) \neq P(h\vert a')$, which potentially leads to $P_{\mathcal{T}_a}(h) = P_{\mathcal{D}_{tr}}(h_i\vert a_i\neq a) \neq P_{\mathcal{D}_{tr}}(h_i \vert a_i=a) = P_{\mathcal{S}_a}(h)$. In counterfactual estimation, we aim at minimizing the estimation error without treatment bias:
\begin{equation}
\label{eq:cf-risk}
\vspace{-1ex}
    \mathcal{E}_{\mathcal{D}}(f_a) = \E_{h\sim P_{\mathcal{D}}(h), y\sim P_{\mathcal{D}}(y[a]\vert h)} \ell (f_a(h), y). 
    % = \E_{a'\sim P(a')}\E_{h\sim P_{\mathcal{D}}(h\vert a), y\sim P_{\mathcal{D}}(y[a]\vert h)}\ell(f_a(h), y).
\end{equation}
Here, we focus on the case where no covariate and concept shifts happen across datasets \footnote{The distribution shift between $P_{\mathcal{D}}(h)$ v.s. $P_{\mathcal{D}_{tr}}(h)$~(covariate shift) and $P_{\mathcal{D}}(y[a]\vert h)$ v.s. $P_{\mathcal{D}_{tr}}(y[a]\vert h)$~(concept shift) can be viewed as the general covariate shift/concept shift and fit into existing theories of domain adaptation/generalization \citep{farahani2021brief}.}: $P_{\mathcal{D}}(h) = P_{\mathcal{D}_{tr}}(h) = \sum_{a'}P(a')P_{\mathcal{D}_{tr}}(h\vert a')$, $P_{\mathcal{D}}(y[a]\vert h) = P_{\mathcal{D}_{tr}}(y[a]\vert h) = P_{\mathcal{S}_a}(y[a]\vert h) = P_{\mathcal{T}_a}(y[a]\vert h)$. Then Eq. \ref{eq:cf-risk}~becomes:
\begin{equation}
\label{eq:all-risk}
\begin{aligned}
     \mathcal{E}_{\mathcal{D}}(f_a) = &\sum_{a'}P(a')\E_{h\sim P_{\mathcal{D}_{tr}}(h\vert a'), y\sim P_{\mathcal{D}_{tr}}(y[a]\vert h)}\ell(f_a(h), y)\\
     =& P(a) \underbrace{\E_{h\sim P_{\mathcal{S}_a}(h), y\sim P_{\mathcal{D}_{tr}}(y[a]\vert h)} \ell(f_a(h), y)}_{\mathcal{E}_{\mathcal{S}_a}(f_a)} +\\
     & (1 - P(a)) \underbrace{\E_{h\sim P_{\mathcal{T}_{a}}(h), y\sim P_{\mathcal{D}_{tr}}(y[a]\vert h)}\ell(f_a(h), y)}_{\mathcal{E}_{\mathcal{T}_a}(f_a)}.
\end{aligned}
\end{equation}

So far we can see that for the task of finding an outcome estimator for treatment type $a$, the counterfactual estimation error is bounded by estimation errors on both $\mathcal{S}_a$ (denoted as $\mathcal{E}_{\mathcal{S}_a}(f_a)$) and $\mathcal{T}_a$ (denoted as $\mathcal{E}_{\mathcal{T}_a}(f_a)$). Per our analysis above, $\mathcal{S}_a$ and $\mathcal{T}_a$ corresponds to the labeled source domain data and the unlabeled target domain data in UDA problems, where potential distribution shifts exist between $\mathcal{S}_a$ and $\mathcal{T}_a$ due to treatment bias. While $\mathcal{E}_{\mathcal{S}_a}(f_a)$ can be optimized with supervised learning using factual data in $\mathcal{S}_a$, we cannot directly optimize $\mathcal{E}_{\mathcal{T}_a}(f_a)$ with labeled data directly. 

Recent works~\citep{thota2021contrastive,sagawa2021extending,park2020joint,wang2021tuta} show that contrastive learning, as an effective self-supervised representation learning method, demonstrates strong transferability in UDA and leads to simple state-of-the-art algorithms. Considering the close connection between counterfactual outcome estimation and UDA, we develop our model \modelshortname~based on contrastive learning and our analysis of the counterfactual estimation error bound from the recent work \citep{haochen2022beyond}, where the authors provide theoretical analysis of the transferability of contrastive learning in UDA. 

We provide the complete proof in Appendix \ref{sec:proof}.
\section{Experiments}
% \todo{
% \begin{enumerate}
%     % \item  \textcolor{green}{\text{[Done]} Experiments with attention along the feature dimension.}
%     % \item  \textcolor{green}{\text{[Done]} Experiments with square inversely penalized prediction errors.}
%     % \item  \textcolor{green}{\text{[Done]} Experiments with non-cold-start examples.}
%     \item \textcolor{orange}{
%     Ablation studies (should finish in 40 hours):
%     \begin{itemize}
%         % \item \textcolor{green}{\text{[Done]} Performance of vanilla Transformer/CausalTransformer encoders.}
%         % \item Rerun ablation study for the contrastive loss options (no contrastive/entire sequence,/component-wise) for the newly proposed backbone.
%         \item Choices of feature position encoding options (non-hierarchical/hierarchical).
%         % \item Choices of stepwise loss weights (uniform/inversed).
%         \item Results comparison of using causal discovery as feature selection.
%     \end{itemize}
%     }
%     \item 2 weaker but also commonly used baselines (MSM, \textcolor{green}{G-Net(Done)}).
%     % \item Visualization of learned history representations. Demonstrate that the learned representations has no treatment bias.
% \end{enumerate}
% }

\paragraph{Datasets.} Following the common evaluation setup for counterfactual treatment outcome estimation overtime in \citep{lim2018forecasting,Bica2020Estimating,melnychuk2022causal}, we use datasets from both synthetic simulation and real-world observation in our experiments. We provide more detailed dataset description in Appendix~\ref{sec:dataset-description}. \textbf{(1) Tumor growth.} Following previous work~\citep{lim2018forecasting,Bica2020Estimating,melnychuk2022causal}, we run the pharmacokinetic-pharmacodynamic(PK-PD) tumor growth simulation and generates a fully-synthetic dataset. The PK-PD simulation~\citep{geng2017prediction} is a state-of-the-art bio-mathematical model simulating the combined effect of chemotherapy and radiotherapy on tumor volumes. \textbf{(2) Semi-synthetic MIMIC-III.} \citep{melnychuk2022causal} constructs a semi-synthetic dataset by simulating outcomes under endogenous dependencies on time and exogenous dependencies on observational patient trajectories that are high dimensional and contain long-range dependencies. We include it in our evaluation as a more challenging synthetic dataset. \textbf{(3) M5.} The M5 Forecasting dataset~\citep{MAKRIDAKIS20221346} contains daily sales of Walmart stores across three US states, along with the metadata of items and stores, as well as explanatory variables such as price and special events. We transform it to treatment outcome estimation task with the treatment variable of item price, and the outcome variable of the sales of items. Covariate variables include all remaining features. With synthetic data, we report the counterfactual outcome estimation errors and compare the performance of \modelshortname~with baselines. However, with real-world data, the counterfactual outcome cannot be observed or simulated, thus, we only report the prediction errors of factual outcome. 

\paragraph{Feature distribution shifts.} To achieve a comprehensive evaluation of counterfactual outcome estimation performance, we introduce feature distribution shifts into the datasets. For each dataset, we split based on static characteristics of subjects into a subset in the source domain and a subset in the target domain. Each subset is further divided into train/validation/test sets. 
%We train each model with data from the source domain and evaluate its performance under 3 settings: (1) \textbf{Non-cold-start setting.} Performance on the test set of the source domain. (2) \textbf{Cold-start setting.} Performance on the test set of the target domain. (3) \textbf{Few-shot setting.} Performance on the test set of the target domain after fine-tuning trained models with few-shot examples from the target domain. 
We summarize the main statistics of datasets in Table~\ref{tab:data-stat} in the appendix.

\paragraph{Baselines.}
We select comprehensive methods for estimating counterfactual outcomes over time as baselines, including MSM~\citep{robins2000marginal}, RMSN~\citep{lim2018forecasting}, CRN~\citep{Bica2020Estimating}, G-Net~\citep{li2021g}, and Causal Transformer (CT)~\citep{melnychuk2022causal}. MSM has difficulty converging when trained with high-dimensional input in semi-synthetic MIMIC-III and M5 datasets and we thus only evaluate it for tumor growth. We empirically find that the balanced representation training losses proposed in CRN and CT do not bring a robust improvement over their variants trained only with the empirical risk minimization (ERM) on factual outcomes. Therefore, we also include these variants, CRN(ERM) and CT(ERM), as baselines.

\subsection{Zero-shot transfer setup}
To showcase cold-start prediction capabilities, in this setup, we focus on the performance on the target domain after training the model in the source domain, with distributional difference in features.
Results are shown in Table \ref{tab:0-shot-merged-results}. \modelshortname~demonstrates the state-of-the-art performance in a majority of horizons across datasets (4/6 in Tumor growth, 6/6 in Semi-synthetic MIMIC-III and M5). On average, \modelshortname~decreases the outcome estimation errors by over 6.2\%, 22.5\% and 26.3\% compared to baselines. Results demonstrate the strong transferability of \modelshortname~in the zero-shot transfer setting.

\begin{table*}[!h]
\centering
\caption{Results of zero-shot transfer setup for multi-step outcome estimation. We report the mean $\pm$ standard deviation of Rooted Mean Squared Errors (RMSEs~$\downarrow$) over 5 runs. \textbf{Bold}: the best results. \underline{Underline}: the 2nd best results.}
\label{tab:0-shot-merged-results}
\resizebox{\textwidth}{!}{

\begin{tabular}{@{}l|l|cccccc|c|c@{}}
\toprule
Dataset & Method & $\tau=1$ & $\tau=2$ & $\tau=3$ & $\tau=4$ & $\tau=5$ & $\tau=6$ & Avg & Gain(\%) \\ \midrule
\multirow{8}{*}{\begin{tabular}[c]{@{}l@{}}Tumor\\ growth\end{tabular}} & MSM & \meanstd{1.0515}{0.0674}&	\textbf{\meanstd{0.5048}{0.0591}}&	\textbf{\meanstd{0.7583}{0.0831}}&	\underline{\meanstd{0.9685}{0.1066}}&	\underline{\meanstd{1.1561}{0.1243}}&	\meanstd{1.3372}{0.1356}& \underline{\meanstd{0.9627}{0.0923}} & 6.2\% \\
& RMSN & \meanstd{1.2406}{0.1301} & \meanstd{1.0914}{0.0346} & \meanstd{1.1315}{0.0634} & \meanstd{1.1583}{0.0810} & \meanstd{1.1674}{0.0913} & \underline{\meanstd{1.1531}{0.0919}} & \meanstd{1.1571}{0.0660} & 22.0\%\\
& CRN(ERM) & \meanstd{1.2924}{0.0772} & \meanstd{1.1769}{0.1058} & \meanstd{1.1728}{0.1136} & \meanstd{1.1906}{0.1106} & \meanstd{1.1997}{0.1061} & \meanstd{1.1883}{0.0985} & \meanstd{1.2035}{0.0901} & 25.0\%\\
% & CRN & \meanstd{1.5643}{0.0667} & \meanstd{10.1978}{14.9899} & \meanstd{14.5460}{23.3314} & \meanstd{17.2080}{28.4348} & \meanstd{18.8712}{31.5756} & \meanstd{19.9238}{33.5369}& \meanstd{13.7185}{21.9748} & 93.4\%\\
& CRN & \meanstd{1.6047}{0.0487}&	\meanstd{2.0846}{0.1665}&	\meanstd{2.0963}{0.1274}	&\meanstd{2.1574}{0.1255}	&\meanstd{2.2609}{0.0569}	&\meanstd{2.3704}{0.1368}	&\meanstd{2.0957}{0.0514} & 68.9\%\\
& CT(ERM) & \underline{\meanstd{0.9729}{0.0718}} & \meanstd{1.0217}{0.0292} & \meanstd{1.1173}{0.0457} & \meanstd{1.1904}{0.0395} & \meanstd{1.2359}{0.0618} & \meanstd{1.2913}{0.0939} & \meanstd{1.1383}{0.0251} & 20.7\%\\
& CT & \meanstd{1.0272}{0.1077} & \meanstd{1.1428}{0.2182} & \meanstd{1.2708}{0.2471} & \meanstd{1.3608}{0.2681} & \meanstd{1.4166}{0.2935} & \meanstd{1.4322}{0.3138} & \meanstd{1.2751}{0.2326} & 29.2\%\\
& G-Net & \meanstd{1.0492}{0.0529} &\meanstd{1.0125}{0.0767}& \meanstd{1.1271}{0.0876} & \meanstd{1.2153}{0.0777} & \meanstd{1.2549}{0.0727} & \meanstd{1.2543}{0.0678} & \meanstd{1.1522}{0.0537} & 21.7\%\\
\cmidrule(l){2-10}
% & \modelshortname & \textbf{\meanstd{0.8748}{0.0655}} & \textbf{\meanstd{0.8566}{0.0903}} & \textbf{\meanstd{0.8766}{0.0287}} & \textbf{\meanstd{0.9757}{0.0621}} & \textbf{\meanstd{1.0872}{0.1042}} & \textbf{\meanstd{1.1206}{0.0868}}\\
& \modelshortname & \textbf{\meanstd{0.8767}{0.0492}}&	\underline{\meanstd{0.7995}{0.0853}}&	\underline{\meanstd{0.8282}{0.0801}}&	\textbf{\meanstd{0.9021}{0.1062}}&	\textbf{\meanstd{0.9888}{0.1280}}&	\textbf{\meanstd{1.0210}{0.1168}} & \textbf{\meanstd{0.9027}{0.0814}} & (-)\\ 
% & Rel. Diff & \textcolor{Green}{-9.89\%}&	\textcolor{Brown}{58.38\%}&	\textcolor{Brown}{9.22\%}&	\textcolor{Green}{-6.86\%}&	\textcolor{Green}{-14.47\%}&	\textcolor{Green}{-11.46\%} \\
\midrule
\multirow{8}{*}{\begin{tabular}[c]{@{}l@{}}Semi-\\synthetic\\MIMIC-III\end{tabular}} 
% & MSM \\
& RMSN & \meanstd{0.2551}{0.0303} & \meanstd{0.6641}{0.1092} & \meanstd{0.9107}{0.1915} & \meanstd{1.1217}{0.2916} & \meanstd{1.2821}{0.3603} & \meanstd{1.3950}{0.4038} & \meanstd{0.9381}{0.2210} & 44.7\%\\
& CRN(ERM) & \underline{\meanstd{0.2506}{0.0303}} & \meanstd{0.5545}{0.0917} & \meanstd{0.7581}{0.1112} & \meanstd{0.9018}{0.1547} & \meanstd{1.0113}{0.1941} & \meanstd{1.1068}{0.2324} & \meanstd{0.7639}{0.1238} & 32.1\%\\
& CRN & \meanstd{0.4041}{0.0537} & \meanstd{0.8256}{0.1767} & \meanstd{1.0439}{0.1958} & \meanstd{1.1807}{0.1725} & \meanstd{1.3121}{0.2229} & \meanstd{1.4374}{0.3089} & \meanstd{1.0340}{0.1606} & 49.8\%\\
& CT(ERM) & \meanstd{0.2762}{0.0804} & \underline{\meanstd{0.5397}{0.1181}} & \underline{\meanstd{0.6765}{0.1417}} & \underline{\meanstd{0.7728}{0.1636}} & \underline{\meanstd{0.8451}{0.1850}} & \underline{\meanstd{0.9028}{0.2070}} & \underline{\meanstd{0.6688}{0.1472}} & 22.5\%\\
& CT & \meanstd{0.3138}{0.0458} & \meanstd{0.5992}{0.0492} & \meanstd{0.7576}{0.0694} & \meanstd{0.8695}{0.0921} & \meanstd{0.9510}{0.1118} & \meanstd{1.0128}{0.1274} & \meanstd{0.7506}{0.0797} & 30.9\%\\
& G-Net & \meanstd{0.5514}{0.1502} & \meanstd{0.9398}{0.2384} & \meanstd{1.2461}{0.3321} & \meanstd{1.4985}{0.4024} & \meanstd{1.7045}{0.4463} & \meanstd{1.8731}{0.4660} & \meanstd{1.3022}{0.3367} & 60.2\%\\ \cmidrule(l){2-10}
% & \modelshortname & \textbf{\meanstd{0.2357}{0.0184}} & \textbf{\meanstd{0.4656}{0.0873}} & \textbf{\meanstd{0.5547}{0.0998}} & \textbf{\meanstd{0.6094}{0.1039}} & \textbf{\meanstd{0.6440}{0.1059}} & \textbf{\meanstd{0.6681}{0.1070}} \\ \midrule
& \modelshortname & \textbf{\meanstd{0.2266}{0.0249}}&\textbf{\meanstd{0.4501}{0.0893}}&\textbf{\meanstd{0.5406}{0.0987}}&\textbf{\meanstd{0.5964}{0.1020}}&\textbf{\meanstd{0.6344}{0.1040}}& \textbf{\meanstd{0.6637}{0.1052}} & \textbf{\meanstd{0.5186}{0.0869}} & (-)\\
% & Rel. Diff & \textcolor{Green}{-9.58\%}&	\textcolor{Green}{-16.60\%}&	\textcolor{Green}{-20.09\%}&	\textcolor{Green}{-22.83\%}&	\textcolor{Green}{-24.93\%}&	\textcolor{Green}{-26.48\%} \\
\midrule

\multirow{8}{*}{M5} 
% & MSM \\
& RMSN & \meanstd{15.1616}{2.0027} & \meanstd{13.9966}{0.5316} & \meanstd{13.4899}{1.2632} & \meanstd{13.5162}{1.7437} & \meanstd{13.8004}{2.0637} & \meanstd{14.3366}{2.3891} & \meanstd{13.8280}{1.5526} & 47.5\%\\
& CRN(ERM) & \meanstd{9.8859}{1.2980} & \meanstd{20.8199}{3.9049} & \meanstd{38.2653}{8.9897} & \meanstd{59.4192}{16.2788} & \meanstd{82.9515}{26.1928} & \meanstd{105.8120}{35.5325} & \meanstd{61.4536}{17.5760} & 88.2\%\\
& CRN & \meanstd{8.1119}{0.3183} & \meanstd{10.3741}{2.2616} & \meanstd{12.9356}{3.1588} & \meanstd{15.4168}{3.8002} & \meanstd{18.1382}{4.6750} & \meanstd{21.1337}{5.4694}& \meanstd{15.5997}{3.7687} & 53.5\%\\
& CT(ERM) & \meanstd{7.1253}{0.5777} & \meanstd{8.3438}{1.0313} & \meanstd{9.2014}{1.4146} & \meanstd{9.9409}{1.7572} & \meanstd{10.6726}{2.1718} & \meanstd{11.3597}{2.5966}& \meanstd{9.9037}{1.7852} & 26.7\%\\
& CT & \underline{\meanstd{7.1239}{0.5770}} & \underline{\meanstd{8.2939}{0.9702}} & \underline{\meanstd{9.1465}{1.3397}} & \underline{\meanstd{9.9091}{1.7198}} & \underline{\meanstd{10.6311}{2.0328}} & \underline{\meanstd{11.3032}{2.4185}}& \underline{\meanstd{9.8568}{1.6959}} & 26.3\%\\ 
& G-Net & \meanstd{7.5358}{0.1605}&	\meanstd{8.6077}{0.3166}&	\meanstd{9.7167}{0.4861}&	\meanstd{10.8993}{0.6902}&	\meanstd{12.3477}{0.8940}&	\meanstd{13.8200}{1.1193}& \meanstd{10.4879}{0.6078}& 30.8\%  \\ \cmidrule(l){2-10}
% & \modelshortname & \textbf{\meanstd{6.3854}{0.0416}} & \textbf{\meanstd{6.9033}{0.0261}} & \textbf{\meanstd{7.2103}{0.0292}} & \textbf{\meanstd{7.4331}{0.0314}} & \textbf{\meanstd{7.6730}{0.0308}} & \textbf{\meanstd{7.8015}{0.0427}} \\
& \modelshortname & \textbf{\meanstd{6.4054}{0.0547}}& \textbf{\meanstd{6.9328}{0.0634}}& \textbf{\meanstd{7.2428}{0.0700}}& \textbf{\meanstd{7.4585}{0.0580}}& \textbf{\meanstd{7.7012}{0.0627}}& \textbf{\meanstd{7.8278}{0.0651}}& \textbf{\meanstd{7.2614}{0.0609}}& (-) \\
% & Rel. Diff & \textcolor{Green}{-10.09\%}&	\textcolor{Green}{-16.41\%}&	\textcolor{Green}{-20.81\%}&	\textcolor{Green}{-24.73\%}&	\textcolor{Green}{-27.56\%}&	\textcolor{Green}{-30.75\%}\\
\bottomrule
\end{tabular}

}
\end{table*}

% \begin{table}[!h]
% \centering
% \label{tab:0-shot-result-semisyn-mimic-2}
% \resizebox{0.857\textwidth}{!}{
% \begin{tabular}{@{}lccccc@{}}
% \toprule
% Method & $\tau=7$ & $\tau=8$ & $\tau=9$ & $\tau=10$ & $\tau=11$ \\ \midrule
% MSM \\
% RMSN & 1.4765$\pm$0.4352 & 1.5383$\pm$0.4563 & 1.5793$\pm$0.4607 & 1.6141$\pm$0.4614 & 1.6463$\pm$0.4608\\
% CRN(ERM) & 1.1913$\pm$0.2715 & 1.2724$\pm$0.3143 & 1.3375$\pm$0.3548 & 1.3940$\pm$0.3894 & 1.4440$\pm$0.4163\\
% CRN & 1.5544$\pm$0.3875 & 1.6603$\pm$0.4461 & 1.7380$\pm$0.4896 & 1.7905$\pm$0.5156 & 1.8236$\pm$0.5317\\
% CT(ERM) & 0.9509$\pm$0.2288 & 0.9932$\pm$0.2504 & 1.0293$\pm$0.2688 & 1.0608$\pm$0.2851 & 1.0882$\pm$0.2996\\
% CT & 1.0610$\pm$0.1401 & 1.1013$\pm$0.1506 & 1.1357$\pm$0.1586 & 1.1655$\pm$0.1645 & 1.1909$\pm$0.1690\\
% G-Net \\ \midrule
% \modelshortname & \textbf{0.6850$\pm$0.1041} & \textbf{0.6969$\pm$0.1025} & \textbf{0.7088$\pm$0.0995} & \textbf{0.7216$\pm$0.0971} & \textbf{0.7372$\pm$0.0964} \\
% \bottomrule
% \end{tabular}
% }
% \end{table}

\subsection{Data-efficient transfer learning setup}
Effectively utilizing small amount of target domain data can be important, and we showcase that it is indeed one of the key strengths of the proposed approach.

To demonstrate this, for the Tumor Growth Dataset, we fine-tune each method trained on the source domain with 100 sequences from the target domain. For the semi-synthetic MIMIC-III and M5 datasets, we set the number of target domain samples for fine-tuning to be 10\% of the number of samples of the target domain in the original dataset. To achieve a fair comparison, we fine-tune each method until it reaches the lowest factual outcome estimation error on a separate validation set in the target domain.

Table~\ref{tab:few-shot-merged-results} compares the performance of all methods in data-efficient transfer learning setup. For the majority of horizons (4/6 in Tumor growth, 5/6 in Semi-synthetic MIMIC-III and M5), we observe that \modelshortname~achieves the state-of-the-art performance after fine-tuning. Again, \modelshortname~reduces the outcome estimation errors by at least 7.8\%, 9.9\% and 4.11\% in the three datasets respectively.
% We also notice that \modelshortname~shows inferior performance in 1-step estimation.
\begin{table*}[!h]
\centering
\caption{Results of the data-efficient transfer learning setup for multi-step outcome estimation. We report the mean $\pm$ standard deviation of Rooted Mean Squared Errors (RMSEs~$\downarrow$) over 5 runs. \textbf{Bold}: the best results. \underline{Underline}: the 2nd best results.}
\label{tab:few-shot-merged-results}
\resizebox{\textwidth}{!}{

\begin{tabular}{@{}l|l|cccccc|c|c@{}}
\toprule
Dataset & Method & $\tau=1$ & $\tau=2$ & $\tau=3$ & $\tau=4$ & $\tau=5$ & $\tau=6$ & Avg & Gain(\%) \\ \midrule
\multirow{8}{*}{\begin{tabular}[c]{@{}l@{}}Tumor\\ growth\end{tabular}} & MSM & \meanstd{1.0436}{0.0671}&	\textbf{\meanstd{0.5023}{0.0588}}&	\textbf{\meanstd{0.7475}{0.0829}}&	\underline{\meanstd{0.9537}{0.1060}}&	\meanstd{1.1376}{0.1233}&	\meanstd{1.3146}{0.1338} & \underline{\meanstd{0.9499}{0.0915}} & 7.8\% \\
& RMSN & \meanstd{1.1839}{0.0842} & \meanstd{1.0912}{0.0405} & \meanstd{1.1215}{0.0593} & \meanstd{1.1538}{0.0688} & \meanstd{1.1728}{0.0773} & \meanstd{1.1740}{0.0830} & \meanstd{1.1495}{0.0529} & 23.8\%\\
& CRN(ERM) & \meanstd{1.2648}{0.0689} & \meanstd{1.1740}{0.1015} & \meanstd{1.1507}{0.1016} & \meanstd{1.1474}{0.1070} & \meanstd{1.1414}{0.1041} & \meanstd{1.1203}{0.0906} & \meanstd{1.1664}{0.0786} & 24.9\%\\
% & CRN & \meanstd{1.4721}{0.0620} & \meanstd{8.4730}{13.1786} & \meanstd{13.1474}{21.7125} & \meanstd{17.2385}{29.3990} & \meanstd{21.5399}{37.7228} & \meanstd{25.3391}{45.1447} & \meanstd{14.5350}{24.5243} & 94.0\%\\
& CRN & \meanstd{1.5019}{0.0587}	&\meanstd{1.5362}{0.0248}	&\meanstd{1.7824}{0.1060}	&\meanstd{1.9842}{0.1707} &\meanstd{2.1317}{0.2431}	&\meanstd{2.2546}{0.3414}	&\meanstd{1.8651}{0.1318} & 53.0\%\\
& CT(ERM) & \underline{\meanstd{0.8947}{0.0668}} & \meanstd{0.8700}{0.0857} & \meanstd{0.9507}{0.1309} & \meanstd{1.0031}{0.1502} & \underline{\meanstd{1.0371}{0.1545}} & \underline{\meanstd{1.0668}{0.1565}} & \meanstd{0.9704}{0.1098} & 9.7\%\\
& CT & \meanstd{0.9545}{0.0782} & \meanstd{0.9494}{0.1597} & \meanstd{1.0225}{0.1562} & \meanstd{1.1062}{0.1377} & \meanstd{1.1455}{0.1192} & \meanstd{1.1562}{0.0953} & \meanstd{1.0557}{0.1136} & 17.0\%\\
& G-Net & \meanstd{1.0335}{0.0622} & \meanstd{1.0154}{0.1100} & \meanstd{1.1105}{0.1476} & \meanstd{1.1859}{0.1620} & \meanstd{1.2257}{0.1693} & \meanstd{1.2198}{0.1508} & \meanstd{1.1318}{0.1118} & 22.6\%\\ \cmidrule(l){2-10}
% & \modelshortname & \textbf{\meanstd{0.8690}{0.0294}} & \textbf{\meanstd{0.8393}{0.0558}} & \textbf{\meanstd{0.8550}{0.0416}} & \textbf{\meanstd{0.8899}{0.0422}} & \textbf{\meanstd{0.9539}{0.0650}} & \textbf{\meanstd{0.9884}{0.0825}} \\ 
& \modelshortname & \textbf{\meanstd{0.8654}{0.0328}}&	\underline{\meanstd{0.7945}{0.0532}}&	\underline{\meanstd{0.8248}{0.0751}}&	\textbf{\meanstd{0.8754}{0.0987}}&	\textbf{\meanstd{0.9378}{0.1176}}&	\textbf{\meanstd{0.9594}{0.1062}} & \textbf{\meanstd{0.8762}{0.0720}} & (-)\\ 
% & Rel. Diff & \textcolor{Green}{-3.27\%}&	\textcolor{Brown}{58.17\%}&	\textcolor{Brown}{10.34\%}&	\textcolor{Green}{-8.21\%}&	\textcolor{Green}{-9.57\%}&	\textcolor{Green}{-10.07\%} \\
\midrule
\multirow{8}{*}{\begin{tabular}[c]{@{}l@{}}Semi-\\synthetic\\MIMIC-III\end{tabular}} 
% & MSM \\
& RMSN & \underline{\meanstd{0.2100}{0.0192}} & \meanstd{0.6084}{0.1114} & \meanstd{0.7745}{0.1180} & \meanstd{0.8908}{0.1402} & \meanstd{0.9776}{0.1505} & \meanstd{1.0440}{0.1529} & \meanstd{0.7509}{0.1123} & 31.0\%\\
& CRN(ERM) & \textbf{\meanstd{0.1946}{0.0158}} & \underline{\meanstd{0.4770}{0.0808}} & \underline{\meanstd{0.5983}{0.0923}} & \underline{\meanstd{0.6786}{0.1004}} & \underline{\meanstd{0.7315}{0.1047}} & \underline{\meanstd{0.7690}{0.1070}} & \underline{\meanstd{0.5748}{0.0823}} & 9.9\%\\
& CRN & \meanstd{0.2955}{0.0256} & \meanstd{0.5051}{0.0748} & \meanstd{0.6361}{0.0786} & \meanstd{0.7277}{0.0783} & \meanstd{0.7919}{0.0764} & \meanstd{0.8379}{0.0759} & \meanstd{0.6324}{0.0656} & 18.1\%\\
& CT(ERM) & \meanstd{0.2704}{0.0631} & \meanstd{0.5347}{0.1061} & \meanstd{0.6712}{0.1252} & \meanstd{0.7679}{0.1433} & \meanstd{0.8402}{0.1607} & \meanstd{0.8968}{0.1784} & \meanstd{0.6635}{0.1279} & 21.9\%\\
& CT & \meanstd{0.3105}{0.0459} & \meanstd{0.5840}{0.0633} & \meanstd{0.7414}{0.0887} & \meanstd{0.8530}{0.1157} & \meanstd{0.9348}{0.1392} & \meanstd{0.9974}{0.1608} & \meanstd{0.7368}{0.0971} & 29.7\%\\
& G-Net & \meanstd{0.3814}{0.0556} & \meanstd{0.6519}{0.0856} & \meanstd{0.8183}{0.1122} & \meanstd{0.9413}{0.1365} & \meanstd{1.0359}{0.1592} & \meanstd{1.1117}{0.1795} & \meanstd{0.8234}{0.1191} & 37.1\%\\ \cmidrule(l){2-10}
% & \modelshortname & \meanstd{0.2921}{0.0306} & \textbf{\meanstd{0.4665}{0.0839}} & \textbf{\meanstd{0.5563}{0.0964}} & \textbf{\meanstd{0.6090}{0.1000}} & \textbf{\meanstd{0.6436}{0.1030}} & \textbf{\meanstd{0.6677}{0.1043}}\\ \midrule
& \modelshortname & \meanstd{0.2288}{0.0229}&	\textbf{\meanstd{0.4496}{0.0877}}&	\textbf{\meanstd{0.5393}{0.0962}}&	\textbf{\meanstd{0.5946}{0.0990}}&	\textbf{\meanstd{0.6326}{0.1013}}&	\textbf{\meanstd{0.6626}{0.1026}}& \textbf{\meanstd{0.5179}{0.0844}} & (-)\\
% & Rel. Diff & \textcolor{Brown}{17.57\%}&	\textcolor{Green}{-5.74\%}&	\textcolor{Green}{-9.86\%}&	\textcolor{Green}{-12.38\%}&	\textcolor{Green}{-13.52\%}&	\textcolor{Green}{-13.84\%} \\
\midrule
\multirow{8}{*}{M5} 
% & MSM & \\
& RMSN & \meanstd{13.9705}{0.3867}&	\meanstd{13.6233}{0.8150}&	\meanstd{13.3291}{1.2900}&	\meanstd{13.1984}{1.3892}&	\meanstd{13.0889}{1.2605}&	\meanstd{13.0108}{1.1173} & \meanstd{13.2501}{1.1696}& 45.92\%\\
& CRN(ERM) & \meanstd{6.3558}{0.0594} & \meanstd{7.0530}{0.0433} & \meanstd{7.3452}{0.0447} & \meanstd{7.5541}{0.0392} & \meanstd{7.7636}{0.0450} & \meanstd{7.9247}{0.0561}& \meanstd{7.5281}{0.0447} & 4.82\%\\
& CRN & \meanstd{6.2868}{0.0471} & \meanstd{7.0282}{0.0482} & \meanstd{7.3327}{0.0610} & \underline{\meanstd{7.5378}{0.0521}} & \underline{\meanstd{7.7492}{0.0586}} & \underline{\meanstd{7.9094}{0.0676}}& \meanstd{7.5115}{0.0572} & 4.60\%\\
% & CT(ERM) & \textbf{\meanstd{6.1759}{0.0386}} & \underline{\meanstd{6.9169}{0.0557}} & \underline{\meanstd{7.2637}{0.0867}} & \underline{\meanstd{7.5102}{0.1128}} & \underline{\meanstd{7.7432}{0.1110}} & \underline{\meanstd{7.8944}{0.1183}}\\
& CT(ERM) & \textbf{\meanstd{6.1720}{0.0354}}&	\underline{\meanstd{6.9309}{0.0571}}&	\underline{\meanstd{7.2855}{0.0889}}&	\meanstd{7.5418}{0.1191}&	\meanstd{7.7839}{0.1283}&	\meanstd{7.9425}{0.1430}&	\meanstd{7.4969}{0.1058}& 4.42\%\\
% & CT & \underline{\meanstd{6.2140}{0.0230}} & \meanstd{7.0108}{0.0294} & \meanstd{7.3617}{0.0552} & \meanstd{7.6184}{0.0975} & \meanstd{7.8602}{0.1246} & \meanstd{8.0322}{0.1537}\\
& CT & \underline{\meanstd{6.2041}{0.0252}}&	\meanstd{7.0022}{0.0372}&	\meanstd{7.3675}{0.0513}&	\meanstd{7.6394}{0.0894}&	\meanstd{7.8932}{0.1153}&	\meanstd{8.0701}{0.1456}&	\meanstd{7.5945}{0.0845}& 5.65\%\\
& G-Net & \meanstd{6.7077}{0.1006}& \meanstd{7.0479}{0.1069}& \meanstd{7.3872}{0.1349}& \meanstd{7.6545}{0.1596}& \meanstd{7.9188}{0.1800}& \meanstd{8.1186}{0.2058}& \underline{\meanstd{7.4725}{0.1461}}& 4.11\% \\
\cmidrule(l){2-10}
% & \modelshortname & \meanstd{6.2851}{0.0430} & \textbf{\meanstd{6.8097}{0.0192}} & \textbf{\meanstd{7.1149}{0.0269}} & \textbf{\meanstd{7.3358}{0.0316}} & \textbf{\meanstd{7.5757}{0.0242}} & \textbf{\meanstd{7.7095}{0.0248}} \\
& \modelshortname & \meanstd{6.3026}{0.0519}&	\textbf{\meanstd{6.8364}{0.0560}}&	\textbf{\meanstd{7.1464}{0.0674}}&	\textbf{\meanstd{7.3634}{0.0619}}&	\textbf{\meanstd{7.6058}{0.0640}}&	\textbf{\meanstd{7.7393}{0.0637}}& \textbf{\meanstd{7.1656}{0.0592}}& (-) \\
% & Rel. Diff & \textcolor{Brown}{1.43\%}&	\textcolor{Green}{-1.16\%}&	\textcolor{Green}{-1.61\%}&	\textcolor{Green}{-1.95\%}&	\textcolor{Green}{-1.77\%}&	\textcolor{Green}{-1.96\%}\\
\bottomrule
\end{tabular}

}
\end{table*}

% \vspace{-6ex}
% \begin{table}[!h]
% \centering
% \label{tab:few-shot-result-semisyn-mimic-2}
% \resizebox{0.857\textwidth}{!}{
% \begin{tabular}{@{}lccccc@{}}
% \toprule
% Method & $\tau=7$ & $\tau=8$ & $\tau=9$ & $\tau=10$ & $\tau=11$ \\ \midrule
% MSM \\
% RMSN & 1.0982$\pm$0.1529 & 1.1428$\pm$0.1514 & 1.1720$\pm$0.1408 & 1.1992$\pm$0.1329 & 1.2267$\pm$0.1296\\
% CRN(ERM) & 0.7929$\pm$0.1064 & 0.8178$\pm$0.1087 & 0.8363$\pm$0.1091 & 0.8547$\pm$0.1096 & 0.8734$\pm$0.1101\\
% CRN & 0.8676$\pm$0.0752 & 0.8936$\pm$0.0796 & 0.9134$\pm$0.0820 & 0.9312$\pm$0.0855 & 0.9484$\pm$0.0895\\
% CT(ERM) & 0.9424$\pm$0.1957 & 0.9817$\pm$0.2132 & 1.0135$\pm$0.2269 & 1.0409$\pm$0.2389 & 1.0644$\pm$0.2494\\
% CT  & 1.0459$\pm$0.1816 & 1.0868$\pm$0.2011 & 1.1203$\pm$0.2173 & 1.1477$\pm$0.2298 & 1.1702$\pm$0.2394\\
% G-Net \\ \midrule
% \modelshortname  & \textbf{0.6870$\pm$0.1018} & \textbf{0.6991$\pm$0.0991} & \textbf{0.7118$\pm$0.0962} & \textbf{0.7243$\pm$0.0929} & \textbf{0.7405$\pm$0.0915}\\
% \bottomrule
% \end{tabular}
% }
% \end{table}

\subsection{Ablation studies}
\vspace{-1ex}
\label{sec:ablation}
\begin{table*}[htbp]
\centering
\caption{Ablation studies for multi-step outcome estimation. We report the mean $\pm$ standard deviation of Rooted Mean Squared Errors (RMSEs~$\downarrow$) over 5 runs. \textbf{Bold}: the best results.}
\label{tab:abla-all-results}
\resizebox{\textwidth}{!}{

\begin{tabular}{@{}l|l|l|cccccc|c|c@{}}
\toprule
Dataset & Component & Choice & $\tau=1$ & $\tau=2$ & $\tau=3$ & $\tau=4$ & $\tau=5$ & $\tau=6$ & Avg & Gain(\%) \\ \midrule
% \multirow{3}{*}{\begin{tabular}[c]{@{}l@{}}Semi-\\synthetic\\MIMIC-III\end{tabular}} & \multicolumn{2}{c}{\modelshortname} & \textbf{\meanstd{0.2357}{0.0184}}&	\textbf{\meanstd{0.4656}{0.0873}}&	\textbf{\meanstd{0.5547}{0.0998}}&	\textbf{\meanstd{0.6094}{0.1039}}&	\textbf{\meanstd{0.6440}{0.1059}}&	\textbf{\meanstd{0.6681}{0.1070}}\\
\multirow{3}{*}{\begin{tabular}[c]{@{}l@{}}Semi-\\synthetic\\MIMIC-III\end{tabular}} & \multicolumn{2}{c|}{\modelshortname} & \meanstd{0.2266}{0.0249}&\meanstd{0.4501}{0.0893}&\textbf{\meanstd{0.5406}{0.0987}}&\textbf{\meanstd{0.5964}{0.1020}}&\textbf{\meanstd{0.6344}{0.1040}}& \meanstd{0.6637}{0.1052} & \meanstd{0.5186}{0.0869}& (-)\\
\cmidrule(l){2-11}
 & \multirow{4}{*}{Encoder} & \begin{tabular}[c]{@{}l@{}}w/ VT\end{tabular} & \meanstd{0.4897}{0.0888}&	\meanstd{0.6161}{0.1139}&	\meanstd{0.6978}{0.1200}&	\meanstd{0.7428}{0.1217}&	\meanstd{0.7705}{0.1196}&	\meanstd{0.7910}{0.1182}& \meanstd{0.6846}{0.1130}& 24.2\%\\
& & w/ CT & \meanstd{0.3519}{0.0584}&	\meanstd{0.4936}{0.0897}&	\meanstd{0.5762}{0.0967}&	\meanstd{0.6279}{0.1017}&	\meanstd{0.6634}{0.1045}&	\meanstd{0.6874}{0.1029}& \meanstd{0.5667}{0.0919}& 8.5\%\\
& & TB only & \meanstd{0.2729}{0.0409} & \meanstd{0.4711}{0.0836} & \meanstd{0.5607}{0.0937} & \meanstd{0.6160}{0.0995} & \meanstd{0.6553}{0.1038} & \meanstd{0.6831}{0.1043} & \meanstd{0.5432}{0.0856} & 	4.5\%\\
& & FB only & \meanstd{1.1210}{0.0827} & \meanstd{1.1287}{0.0855} & \meanstd{1.1755}{0.1076} & \meanstd{1.2055}{0.1251} & \meanstd{1.2296}{0.1418} & \meanstd{1.2547}{0.1566} & \meanstd{1.1858}{0.1155} &	56.3\%\\

\cmidrule(l){2-11}
& \multirow{1}{*}{FPE} & w/ abs & \meanstd{0.2981}{0.0444} & \meanstd{0.4679}{0.0940} & \meanstd{0.5561}{0.1050} & \meanstd{0.6091}{0.1079} & \meanstd{0.6446}{0.1115} & \meanstd{0.6694}{0.1128}& \meanstd{0.5409}{0.0951} & 4.1\%\\
\cmidrule(l){2-11}
& \multirow{2}{*}{SSL Loss} & none & \meanstd{0.2998}{0.0466} & \meanstd{0.4718}{0.0905} & \meanstd{0.5579}{0.1007} & \meanstd{0.6117}{0.1035} & \meanstd{0.6460}{0.1060} & \meanstd{0.6680}{0.1060}& \meanstd{0.5425}{0.0914} & 4.4\%\\ 
& & w/o comp & \meanstd{0.2884}{0.0421} & \meanstd{0.4603}{0.0898} & \meanstd{0.5475}{0.1032} & \meanstd{0.6013}{0.1077} & \meanstd{0.6353}{0.1079} & \textbf{\meanstd{0.6610}{0.1075}}& \meanstd{0.5323}{0.0926} & 2.6\%\\
\cmidrule(l){2-11}
& \multirow{2}{*}{SupL Loss} & w/ uni & \meanstd{0.2910}{0.0355} & \meanstd{0.4656}{0.0873} & \meanstd{0.5547}{0.0998} & \meanstd{0.6094}{0.1039} & \meanstd{0.6440}{0.1059} & \meanstd{0.6681}{0.1070}& \meanstd{0.5884}{0.1006} & 11.9\%\\
& & w/ sq.inv. & \textbf{\meanstd{0.1968}{0.0148}} & \textbf{\meanstd{0.4456}{0.0815}} & \meanstd{0.5415}{0.0906} & \meanstd{0.6021}{0.0954} & \meanstd{0.6431}{0.0956} & \meanstd{0.6761}{0.0948}& \textbf{\meanstd{0.5175}{0.0780}} & -0.2\%\\
\cmidrule(l){2-11}
& \multirow{1}{*}{Decoder} & w/ autoreg & \meanstd{0.2049}{0.0118}& \meanstd{0.7036}{0.1422}& \meanstd{1.0234}{0.2214}& \meanstd{1.2023}{0.2745}& \meanstd{1.4692}{0.3711}& \meanstd{1.6577}{0.4959}& \meanstd{1.0435}{0.2437}& 50.3\%\\
\midrule
% \multirow{3}{*}{\begin{tabular}[c]{@{}l@{}}M5\end{tabular}} & \multicolumn{2}{c}{\modelshortname} & \meanstd{6.3854}{0.0416}&	\textbf{\meanstd{6.9033}{0.0261}}&	\textbf{\meanstd{7.2103}{0.0292}}&	\textbf{\meanstd{7.4331}{0.0314}}&	\textbf{\meanstd{7.6730}{0.0308}}&	\textbf{\meanstd{7.8015}{0.0427}} \\ \cmidrule(l){2-9}
\multirow{3}{*}{\begin{tabular}[c]{@{}l@{}}M5\end{tabular}} & \multicolumn{2}{c|}{\modelshortname} & \meanstd{6.4054}{0.0547}& \textbf{\meanstd{6.9328}{0.0634}}& \textbf{\meanstd{7.2428}{0.0700}}& \textbf{\meanstd{7.4585}{0.0580}}& \textbf{\meanstd{7.7012}{0.0627}}& \textbf{\meanstd{7.8278}{0.0651}}& \textbf{\meanstd{7.2614}{0.0609}}& (-) \\ \cmidrule(l){2-11}
 & \multirow{4}{*}{Encoder} & \begin{tabular}[c]{@{}l@{}}w/ VT\end{tabular} & \meanstd{17.8226}{4.7807}&	\meanstd{17.6769}{4.5561}&	\meanstd{17.5937}{4.4219}&	\meanstd{17.5279}{4.2936}&	\meanstd{17.4113}{4.1662}&	\meanstd{17.2706}{4.0453}& \meanstd{17.5505}{4.3765}& 58.6\%\\
 & & w/ CT & \meanstd{6.7386}{0.2326}&	\meanstd{7.1911}{0.2474}&	\meanstd{7.4549}{0.2322}&	\meanstd{7.6524}{0.2061}&	\meanstd{7.8488}{0.2021}&	\meanstd{7.9589}{0.2006}& \meanstd{7.4741}{0.2176}& 2.8\%\\
 & & TB only & \meanstd{6.4085}{0.0538} & \meanstd{6.9547}{0.0535} & \meanstd{7.2673}{0.0453} & \meanstd{7.4825}{0.0388} & \meanstd{7.7167}{0.0380} & \meanstd{7.8328}{0.0430} & \meanstd{7.2771}{0.0409} & 	0.2\%\\
 & & FB only & \meanstd{6.8805}{0.0333} & \meanstd{7.6298}{0.0212} & \meanstd{7.9706}{0.0254} & \meanstd{8.1215}{0.0298} & \meanstd{8.3989}{0.0411} & \meanstd{8.5303}{0.0435} & \meanstd{7.9219}{0.0311} &	8.3\%\\
 \cmidrule(l){2-11}
& \multirow{1}{*}{FPE} & w/ abs & \meanstd{6.4089}{0.0693} & \meanstd{6.9648}{0.0617}& \meanstd{7.2776}{0.0528}& \meanstd{7.4834}{0.0430}& \meanstd{7.7214}{0.0421}& \meanstd{7.8479}{0.0370}& \meanstd{7.2840}{0.0486}& 0.3\%\\
\cmidrule(l){2-11}
& \multirow{2}{*}{SSL Loss} & none & \meanstd{6.4296}{0.1193} & \meanstd{6.9434}{0.0796} & \meanstd{7.2548}{0.0699} & \meanstd{7.4744}{0.0753} & \meanstd{7.7117}{0.0728} & \meanstd{7.8429}{0.0817}& \meanstd{7.2761}{0.0827}& 0.2\%\\ 
& & w/o comp & \meanstd{6.4637}{0.0926} & \meanstd{6.9847}{0.0764} & \meanstd{7.2934}{0.0705} & \meanstd{7.5093}{0.0661} & \meanstd{7.7497}{0.0669} & \meanstd{7.8748}{0.0608}& \meanstd{7.3126}{0.0715}& 0.7\%\\
\cmidrule(l){2-11}
& \multirow{2}{*}{SupL Loss} & w/ sq.inv. & \textbf{\meanstd{6.3425}{0.0461}} & \meanstd{6.9760}{0.0523} & \meanstd{7.3170}{0.0538} & \meanstd{7.5366}{0.0614} & \meanstd{7.8015}{0.0658} & \meanstd{7.9439}{0.0736}& \meanstd{7.3196}{0.0583}& 0.8\% \\
& & w/ inv & \meanstd{6.3575}{0.0473} & \meanstd{6.9427}{0.0381} & \meanstd{7.2766}{0.0422} & \meanstd{7.4932}{0.0447} & \meanstd{7.7431}{0.0531} & \meanstd{7.8785}{0.0628} & \meanstd{7.2819}{0.0464}& 0.3\%\\
\cmidrule(l){2-11}
& \multirow{1}{*}{Decoder} & w/ autoreg & \meanstd{6.3572}{0.0621}& $>20$&$>20$&$>20$&$>20$&$>20$ & $>20$ & $>60\%$\\
\bottomrule
\end{tabular}
}
\end{table*}

We conduct ablation studies in the zero-shot transfer setup to validate the design of \modelshortname. We choose the feature-rich datasets: semi-synthetic MIMIC-III and M5 since they contain complex dynamics and thus are more viable for evaluating components capturing temporal and feature-wise interactions.

\underline{\textbf{Encoder.}}
To validate the impact of our proposed encoder architecture, we replace the it with the following variants: (1) Vanilla Transformer \textbf{(w/ VT)}. A vanilla transformer with temporal causal attention, which takes the history with all features concatenated as multivariate time series input. (2) CT \textbf{(w/ CT)}. The encoder architecture proposed by Causal Transformer~\citep{melnychuk2022causal} that concatenates features grouped by covariates/treatments/outcomes into 3 subsets first, then applies self-attention/cross-attention among sequences with each group of features/each pair of feature groups in an alternating way.
(3) Temporal Attention Block only (\textbf{TB only}). The variant that only includes temporal attention blocks but not feature-wise attention blocks. (4) Feature-wise Attention Block only (\textbf{FB only}). The variant that only includes feature-wise attention blocks.

Rows ``Encoder $\vert$ w/VT(w/CT)" in Table~\ref{tab:abla-all-results} demonstrate the superior performance of our proposed encoder architecture. We observe that both methods (CT, \modelshortname) processing features respectively outperform VT that simply concatenates all features, marking the importance of explicitly modeling feature interactions. Moreover, the finer-grained modeling of feature interactions between each pair of features in \modelshortname~further improves the estimation performance compared to the coarser modeling of interactions between feature subsets in CT.

Rows ``Encoder $\vert$ TB only(FB only)'' in Table~\ref{tab:abla-all-results} show that the temporal attention blocks are the most critical for temporal outcome estimation, while the feature-wise attention blocks further boost the performance by 4.5\% and 0.2\% on Semi-synthetic MIMIC-III and M5 datasets respectively.

\underline{\textbf{Feature positional encoding (FPE).}}
We replace the tree-based feature positional encoding with its absolute variant (\textbf{w/abs}): each feature maps to a separate learnable encoding vector. We observe that the tree-based positional encoding has gains of 4.1\% and 0.3\% over the absolute variant in the two datasets respectively.

\underline{\textbf{Self-supervised loss (SSL).}}
To validate the improvement brought by introducing self-supervised learning as well as the choice of its training loss, we compare \modelshortname~with two variants: (i) \textbf{none.} A model with the same architecture as \modelshortname~but trained with factual estimation losses only; and (ii) \textbf{w/o comp.} with vanilla MoCo v3 training loss in Eq.~\ref{eq:overall-con-loss} for self-supervised learning. Rows ``SSL Loss $\vert$ none(w/o comp)" in Table~\ref{tab:abla-all-results} compare the estimation performance of the aforementioned choices of self-supervised learning losses and validate the effectiveness of our component-wise contrastive loss in self-supervised learning.

\underline{\textbf{Supervised loss (SupL).}}
We consider different choices of the hyperparameter in the supervised training loss of Eq. \ref{eq:suploss}: (1) \textbf{w/ uni}: a uniform weight with each $w_i = 1/\tau$; (2) \textbf{w/ inv}: weights in proportion to the inverse of horizon $w_i = \frac{1/i}{\sum_{j=1}^\tau 1/j}$; (3) \textbf{w/ sq.inv.}: weights in proportion to the inverse of the squared horizon $w_i \ =\frac{1/i^2}{\sum_{j=1}^\tau 1/j^2}$. Both (2) and (3) are designed to enhance the short-term outcome estimation performance. We select the weights by validation error for each dataset (w/inv for Semi-synthetic MIMIC-III and w/uni for M5), and compare it to the other two variants. While the relative performance order varies across datasets, all variants can outperform the best baseline results in Table~\ref{tab:0-shot-merged-results}. 

\underline{\textbf{Decoder.}}
We validate the effectiveness of our non-autoregressive design of the decoder and compare it with an autoregressive alternative (\textbf{w/ autoreg}) by including the previous outcome in input features. While results in rows "Decoder $\vert$ w/ autoreg" show good performance in very short horizons ($\tau=1$), multistep outcome estimation errors quickly diverges with the horizon increasing.

\section{Conclusion}
In this work, we propose a self-supervised learning framework - \modelfullname~- to tackle the challenges associated with accurately estimating treatment outcomes over time using observed history, which is a crucial component in areas where randomized controlled trials (RCTs) are not feasible. By integrating self-supervised learning and the Transformer-based encoder combining temporal with feature-wise attention, we've achieved notable advances in estimation accuracy and cross-domain generalization performance.

\bibliography{iclr2024_conference}
\bibliographystyle{iclr2024_conference}

\appendix
% \input{content/2-related-work}

\section{Model Architecture}
\label{sec:detail-encoder}
\paragraph{Input feature projection.}
Assume the concatenation of time-varying variables $(\bar{\mX}_t, \bar{\mA}_{t-1}, \bar{\mY}_t)$ in history as $\bar{\mS}_t \in\mathbb{R}^{t\times d_S}$, $d_S = d_X + d_A + d_Y$, and the static variables $\mV\in\mathbb{R}^{d_V}$. We adopt a linear transformation $f_\text{input}: \mathbb{R} \rightarrow \mathbb{R}^{d_\text{model}}$ to map $\bar{\mS}_t$ and $\mV$ to the embedding space as $\mE^S\in\mathbb{R}^{t\times d_S\times d_\text{model}}$ and $\mE^V\in\mathbb{R}^{d_V\times d_\text{model}}$~respectively, where
\begin{align}
    &\mE^S[i,j] = f_\text{input}(\bar{\mS}_t[i,j]),~1\leq i\leq t,~1\leq j\leq d_S;\\
    &\mE^V[j] = f_\text{input}(\mV[j]),~1\leq j\leq d_V.
\end{align}

\paragraph{Feature positional encoding.}
A shared feature projection function among all features is not sufficient to encode the feature-specific information since the same scalar value represents different semantics in different features. Meanwhile, feature-specific information is also critical for modeling the interactions among features. Therefore we enhance the input embedding with a positional encoding along the feature dimension. Since the features can be grouped into covariates, treatments and outcomes and form a hierarchical structure with 2 levels, we model it with a learnable tree positional encoding~\citep{shiv2019novel,wang2021tuta}. Denote the lists of covariate, treatment, outcome, and static features as $\mF_X$, $\mF_A$, $\mF_Y$, and $\mF_V$. For the $i$-th feature $f_i^\mF$ in a certain feature list $\mF\in\{\mF_X, \mF_A, \mF_Y, \mF_V\}$, its positional encoding is:
\begin{equation}
\begin{array}{l}
    \mE_{\text{fea\_pos}}(f_i^\mF) = \mE_{\text{fea}}\cdot\text{Concat}(\ve_\mF, \ve_i),~\text{where}\\
    \ve_\mF = \left\lbrace
    \begin{array}{l}
    (1, 0, 0, 0)~\text{if $\mF$ is $\mF_X$}\\
    (0, 1, 0, 0)~\text{if $\mF$ is $\mF_A$}\\
    (0, 0, 1, 0)~\text{if $\mF$ is $\mF_Y$}\\
    (0, 0, 0, 1)~\text{if $\mF$ is $\mF_V$}
    \end{array}
    \right.,
\end{array}
\end{equation}
$\ve_i\in\mathbb{R}^{\max(d_X, d_A, d_Y, d_V)}$ is a one-hot vector with only $\ve_i[i] = 1$. $\mE_\text{fea}\in\mathbb{R}^{d_\text{model}\times (4 + \max(d_X, d_A, d_Y, d_V))}$ are learnable tree embedding weights. After obtaining the stacked feature positional embeddings $\mE_\text{fea\_pos}^S\in\mathbb{R}^{d_S\times d_\text{model}}$, $\mE_\text{fea\_pos}^V\in\mathbb{R}^{d_V\times d_\text{model}}$ of time-varying and static features, we broadcast them to the shape of $\mE^S$ and $\mE^V$ respectively along the time dimension. The embedded input is then the sum of input feature projection and feature positional encoding:
\begin{align}
    \mZ^{S,(0)} &= \mE^S + \text{Broadcast}(\mE^S_{\text{fea\_pos}}),\\
    \mZ^{V,(0)} &= \mE^V + \text{Broadcast}(\mE^V_{\text{fea\_pos}}).
\end{align}

\paragraph{Temporal attention block.}
The temporal attention block is designed to capture the temporal dependencies within each feature. We construct the block based on the self-attention part in the conventional Transformer decoder~\citep{vaswani2017attention}.
% , where a multi-head attention module with temporal causal attention is followed by a point-wise feed-forward module, and each module has residual connections with layer normalization. 
Considering the importance of relative time interval in modeling treatment outcomes, we adopt the relative positional encoding~\citep{shaw2018self,melnychuk2022causal} along the time dimension.

The temporal attention block in the $l$-th layer receives $\mZ^{S,(l-1)}$ from the previous layer, reshapes it to $d_S$ sequences with lengths $t$, and passes them through the block in parallel. The outputs $\mZ^{S, (l)}_{\text{tmp}}$ have the same shape as $\mZ^{S,(l-1)}$. Since $\mZ^{V,(l-1)}$ is static, we only pass it through the point-wise feed-forward module and get $\mZ^{V, (l)}_{\text{tmp}}$

\paragraph{Feature-wise attention block.}
The feature-wise attention block models interactions among different features. We reuse the architecture of the conventional Transformer encoder but replace the positional encoding with the feature positional encoding as described. The block in the $l$-th layer receives $\mZ^{S, (l)}_{\text{tmp}}$ from the temporal attention block and reshapes it to $t$ sequences, each with a length $d_S$. We broadcast $\mZ^{V, (l)}_{\text{tmp}}$ and concatenate it with each sequence along the feature dimension to enable the attention among both time-varying and static features. The concatenated $t$ sequences that we apply attention to are:
\begin{equation}
\begin{aligned}
    \mZ^{SV,(l)}_\text{tmp} &= \text{Concat}(\mZ^{S, (l)}_{\text{tmp}}, \text{Broadcast}(\mZ^{V, (l)}_{\text{tmp}}))\\
    &\in\mathbb{R}^{t\times (d_S + d_V)\times d_\text{model}}.
\end{aligned}
\end{equation}
We apply full attention across all features and get $\mZ^{SV, (l)}$ with the same shape as $\mZ^{SV,(l)}_\text{tmp}$. The propagated embeddings of time-varying features are obtained as:
\begin{equation}
    \mZ^{S,(l)} = \mZ^{SV,(l)}[:, :d_S, :]\in\mathbb{R}^{t\times d_S\times d_\text{model}}.
\end{equation}
To keep the $\mZ^{V,(l)}$ static after feature-wise attention, we only propagate $\mZ^{V,(l)}_\text{tmp}$ with full attention among static features only. The updated embeddings of the static features $\mZ^{V,(l)}\in\mathbb{R}^{d_V\times d_\text{model}}$ have the same shape as $\mZ^{V,(l)}_\text{tmp}$.

\section{Identifiability assumptions}
\label{sec:identifiability}
\begin{assumption}[Consistency]
\label{asm:consistency}
The potential outcome of any treatment $\va_t$ is always the same as the factual outcome when a subject is given the treatment $\va_t$: $\vy_{t+1}[\va_t] = \vy_{t+1}$.
\end{assumption}

\begin{assumption}[Positivity]
If $P(\bar{\mA}_{t-1} = \bar{\va}_{t-1}, \bar{\mX}_t = \bar{\vx}_t) \neq 0$, then $P(\mA_t = \va_t \vert \bar{\mA}_{t-1} = \bar{\va}_{t-1}, \bar{\mX}_t = \bar{\vx}_t) > 0$ for any $\bar{\va}_t$.
\end{assumption}

\begin{assumption}[Sequential strong ignorability]
$\mY_{t+1}[\va_t] \indep \mA_t \vert \bar{\mA}_{t-1}, \bar{\mX}_t, \forall \va_t, t$.
\end{assumption}

\section{Why use temporally causal attention in \modelshortname?}
% \paragraph{Benefits of satisfying temporal causality.}
The embeddings of time-varying features $\mZ^{S,(L)}\in\mathbb{R}^{t\times (d_X+d_A+d_Y)\times d_\text{model}}$ from the final layer is re-organized to stepwise representations $(\mZ_1, \mZ_2, \dots, \mZ_t)$. Each $\mZ_{t'}=\{\vz_{t'}^i\in\mathbb{R}^{d_\text{model}}\}_{i=1}^{d_X+d_A+d_Y}$ is further aggregated to $\vz^X_{t'}, \vz^A_{t'}, \vz^Y_{t'}, \vz_{t'}$ in Equation~\ref{eq:rep-agg}. When the encoder satisfies the temporal causality (i.e. $\mZ_{t'}$ only depends on $\bar{\mH}_{t'}$), they can be seen as a sequence of representations for the observed history $\bar{\mH}_1, \bar{\mH}_2, \dots, \bar{\mH}_t$ truncated at each time step.

When we feed the encoder with factual data in training stages, a major advantage of encoders satisfying temporal causality is that we can estimate the outcomes and evaluate the factual estimation losses in every time step of the input sequence at a single forward pass. Evaluating of counterfactual data, the encoder only needs to keep $\mZ_t$ representing the entire observed history, conditioning on which the predictor rolls out outcome estimations given counterfactual treatments.

In contrast, feeding the entire history in one pass for training is error-prone for architectures and can violate the temporal causality (e.g. transformers with fully temporal attention or frequency-based methods), since it leaks information of future steps into the representations in previous steps. Predictors trained with such representations converge quickly to a trivial model that simply copies future steps as estimations. When we evaluate counterfactual data where future counterfactual outcomes are no longer available in input, the performance degenerates.
As a result, we have to explicitly unroll the observed sequence to $t$ truncated sequences and run $t$ forward passes to get the representations and factual errors in every step when training non-temporally-causal models. This leads to a $\times T$  increase in training time when the batch size remains unchanged due to hardware restrictions, where $T$ is the maximum length of sequences in training data.
We have $T\geq 50$ in our experiments and find that none of the architecture violating temporal causality can finish training in a reasonable time.

\section{Proof of generalization bound of contrastive learning in counterfactual outcome estimation}
\label{sec:proof}

\subsection{Preliminaries}
\paragraph{Positive pairs.}
Pairs of semantically related/similar data samples are positive pairs in contrastive learning. In contrastive learning, positive pairs are commonly generated by applying randomized transformation on the same input~\citep{he2020momentum,woo2022cost}.

For exposition simplicity, we assume the set of factual observed history $\mathcal{H}$ is a finite but large dataset of size $N$. We use $P_+$ to denote the distribution of positive pairs. $P_+$ satisfies $P_+(h,h') = P_+(h',h),~\forall~h,h'\in\mathcal{H}$. $P_{\mathcal{H}}$ denotes the marginal distribution of $P_+$: $P_{\mathcal{H}}(h) = \sum_{h'\in\mathcal{H}}P_+(h, h')$.

\paragraph{Positive-pair graph.} Following the definition in \citep{haochen2022beyond}, we introduce the \textit{postive-pair graph} as a weighted undirected graph $G(\mathcal{H}, w)$ with the vertex set $\mathcal{H}$ and the edge weight $w(h, h') = P_+(h, h')$. $w(h) = P_{\mathcal{H}}(h)=\sum_{h'\in\mathcal{H}}w(h, h')$. For any vertex subset $A$, $w(A)=\sum_{h\in A}w(h)$. For any vertex subsets $A,B$, $w(A, B) = \sum_{h\in A, h'\in B}w(h, h')$. For any vertex $h$ and vetex subset $B$, $w(h, B) = w(\{h\}, B)$.

\paragraph{Generalized spectral contrastive loss.}
Let $r: \mathcal{H}\rightarrow \mathbb{R}^k$ be a mapping from the input data to $k$-dimensional features. For the convenience of proof, we consider the (generalized) spectral contrastive loss proposed in \citep{haochen2022beyond}:
\begin{equation}
    \mathcal{L}_\sigma(r) = \E_{(h,h^+)\sim P_+}\left[ \llnorm{r(h) - r(h^+)}^2 \right] + \sigma \cdot R(r),
\end{equation}
where the regularizer is defined as $R(r) = \left\lVert \E_{h\in P_{\mathcal{H}}} [r(h)r(h)^T] - I_{k\times k} \right\rVert_F^2$ and $I_{k\times k}$ is the $k$-dimensional identity matrix. Notice that the InfoNCE loss is more commonly used in empirical study \citep{he2020momentum,chen2020simple,chen2021mocov3} instead of the spectral contrastive loss, and their equivalence is still an open problem with some preliminary results~\citep{tan2023contrastive}. 

\subsection{Definitions and assumptions}
We reiterate the following definitions and assumptions in \citep{haochen2022beyond} for self-containment:
% \subsubsection{Clustering structure of the positive-pair graph}
\begin{definition}[Expansion]
Let $A,B$ be two disjoint subsets of $\mathcal{H}$. We denote the expansion, max-expansion and min-expansion from $A$ to $B$ as follows:
\begin{equation}
\begin{aligned}
    \phi(A, B) &= \frac{w(A, B)}{w(A)},\\
    \bar{\phi}(A, B) &= \max_{h\in A}\frac{w(h, B)}{w(h)},\\
    \underline{\phi}(A, B) &= \min_{h\in A}\frac{w(h, B)}{w(h)}.
\end{aligned}
\end{equation}
\end{definition}

\begin{assumption}[Cross-cluster connections]
\label{asm:cross-cluster}
For some $\alpha\in (0, 1)$, we assume that vertices of the positive-pair graph $G(\mathcal{H}, w)$ can be partitioned into $m$ disjoint clusters $C_1, \dots, C_m$ such that for any $i\in[m]$,
\begin{equation}
    \bar{\phi}(C_i, \mathcal{H} \backslash C_i) \leq \alpha.
\end{equation}
\end{assumption}

\begin{assumption}[Intra-cluster conductance]
\label{asm:intra-cluster}
For all $i\in[m]$, assume the conductance of the subgraph restricted to $C_i$ is large, i.e., every subset A of $C_i$ with at most half the size of $C_i$ expands to the rest:
\begin{equation}
    \forall~A\subset C_i~\text{satisfying}~w(A) \leq w(C_i) / 2,~\phi(A, C_i \backslash A) \geq \gamma.
\end{equation}
\end{assumption}

% \subsubsection{}

\begin{assumption}[Relative expansion]
\label{asm:relative-expansion}
Let $S$ and $T$ be two disjoint subsets of $\mathcal{H}$, each is formed by $r$ clusters among $C_1, C_2, \dots, C_m$ for $r\leq m/2$. Let $\rho = \min_{i\in [r]}\underline{\phi}(T_i, S_i)$ be the minimum min-expansions from $T_i$ to $S_i$. For some sufficiently large universal constant $c$, we assume that $\rho \geq c\cdot \alpha^2$ and that 
\begin{equation}
    \rho = \min_{i\in[r]}\underline{\phi}(T_i, S_i) \geq c\cdot \max_{i\neq j}\bar{\phi}(T_i, S_j).
\end{equation}
\end{assumption}

\subsection{Proof of Theorem \ref{theorem:upper-bound-cont}}
We adapt the preconditioned featurer averaging classifier in \citep{haochen2022beyond} for regression in our proof:
\begin{algorithm}[h]
    \caption{Preconditioned feature averaging (PFA).}
    \label{alg:pfa}
	\textbf{Require:} Pretrained representation extractor $r$, unlabeled data $P_{\mathcal{H}}$, source domain labeled data $P_{S}$, target domain test data $\tilde{h}$, integer $t\in\sZ^+$, outcome discretization granularity $\epsilon$.
	\begin{algorithmic}[1]
	    \State Compute the preconditioner matrix $\Sigma = \E_{h\in P_{\mathcal{H}}}[r(h)r(h)^T]$.
	    \For {every outcome value $y_i$ corresponding to the cluster $C_i, i\in [r]$}
	        \State Compute the mean feature of outcome $y_i$: $b_i = \E_{(h,y)\sim P_S}[\mathbbm{1}[ \llnorm{y - y_i} \leq \epsilon ] \cdot r(h)]$.
	    \EndFor
	    \State \Return prediction $y_{i^*}$, $i^*=\argmax_{i\in [r]}\langle r(h), \sum^{t-1}b_i \rangle$.
	\end{algorithmic}
\end{algorithm}

For any PFA regressor $f$ constructed with Alg. ~\ref{alg:pfa}, we can transform it to a corresponding classifier by defining its 0-1 classification error on the target domain $T$ as:
\begin{equation}
    \mathcal{E}_{T}^{01}(f) = \E_{(h,y)\sim P_T}[\mathbbm{1} [\llnorm{y - f(h)} > \epsilon]].
\end{equation}

We can directly apply the main result in \citep{haochen2022beyond} and get an upper bound of the 0-1 error on the target domain:
\begin{theorem}[Upper bound of 0-1 error on the target domain \citep{haochen2022beyond}]
Suppose that Assumption~\ref{asm:cross-cluster}, Assumption~\ref{asm:intra-cluster}, and Assumption \ref{asm:relative-expansion} holds for the set of observed history $\mathcal{H}$ and its positive-pair graph $G(\mathcal{H}, w)$, and the representation dimension $k\geq 2m$.  Let $r$ be a minimizer of the generalized spectral contrastive loss and the regression head $f$ be constructed in Alg. \ref{alg:pfa}. We have
\begin{equation}
\label{eq:0-1-bound}
    \mathcal{E}^{01}_T(f) \lesssim \frac{r}{\alpha^2\gamma^4}\cdot\exp(-\Omega(\frac{\rho \gamma^2}{\alpha^2})).
\end{equation}
\end{theorem}

\begin{lemma}[Relation between the L2 regression error and 0-1 classification error]
\label{lemma:l2-0-1}
Suppose that both $\llnorm{f(h)}\leq B$ and $\llnorm{y}\leq B,~\epsilon < 2B$. The L2 regression error $\mathcal{E}_T(f)$ of the PFA regressor on the target domain $T$ is bounded by $\mathcal{E}_T^{01}(f)$ as:
\begin{equation}
\label{eq:l2-0-1}
    \mathcal{E}_T(f) \leq \epsilon^2 + (4B^2 - \epsilon^2)\mathcal{E}_T^{01}(f).
\end{equation}
\end{lemma}
\begin{proof}
\begin{equation*}
\begin{aligned}
    \mathcal{E}_T(f) =& \E_{(h,y)\in P_T}\llnorm{y - f(h)}^2\\
    \leq& \sum_{(h,y)\in T} P(h, y)\left[ \mathbbm{1}[\llnorm{y - f(h)} > \epsilon] \llnorm{y - f(h)}^2 \right]\\
    &+ \sum_{(h,y)\in T} P(h, y)(1 - \mathbbm{1}[\llnorm{y - f(h)} > \epsilon])\epsilon^2\\
    \leq& \sum_{(h,y)\in T} P(h, y)\left[ \mathbbm{1}[\llnorm{y - f(h)} > \epsilon] 4B^2 \right]\\
    &+ \sum_{(h,y)\in T} P(h, y)(1 - \mathbbm{1}[\llnorm{y - f(h)} > \epsilon])\epsilon^2\\
    =& \epsilon^2 + (4B^2 - \epsilon^2)\E_{(h,y)\in P_T} \mathbbm{1}[\llnorm{y - f(h)} > \epsilon]\\
    =& \epsilon^2 + (4B^2 - \epsilon^2)\mathcal{E}_T^{01}(f).
\end{aligned}
\end{equation*}
\end{proof}

Lemma \ref{lemma:l2-0-1} connects the L2 error and the 0-1 error. Combining Eq.~\ref{eq:all-risk}, Eq.~\ref{eq:0-1-bound}, Eq.~\ref{eq:l2-0-1}, we immediately get Theorem \ref{theorem:upper-bound-cont}.

\section{Dataset description}
\label{sec:dataset-description}
\begin{table*}[htbp]
\centering
% \vspace{-8ex}
\caption{Statistics of datasets.}
\label{tab:data-stat}
% \resizebox{\linewidth}{!}{
\begin{tabular}{@{}lllcc@{}}
\toprule
Dataset & Domain & Property & Seq Length & Train/Validation/Test Seq Num\\ \midrule
\multirow{2}{*}{\begin{tabular}[c]{@{}l@{}}Tumor\\ growth\end{tabular}} & source & $\gamma=10$  & 60 & 10000/1000/1000 \\ \cmidrule(l){2-5}
& target & $\gamma=0$ & 60 & 100/1000/1000 \\
\midrule
\multirow{2}{*}{\begin{tabular}[c]{@{}l@{}}Semi-synthetic\\MIMIC-III\end{tabular}} & source & age in [20,45] & 99 & 3704/926/926 \\ \cmidrule(l){2-5}
& target & age$\geq$85 & 99 & 138/347/1737 \\
\midrule
\multirow{2}{*}{\begin{tabular}[c]{@{}l@{}}M5\end{tabular}} & source & food items & 50 & 39606/7048/7048 \\ \cmidrule(l){2-5}
& target & household items &  50  & 3623/3512/18005 \\
\bottomrule
\end{tabular}
% }
% \vspace{-2ex}
\end{table*}

We summarize the statistics and the way of introducing feature distribution shifts in Table \ref{tab:data-stat}. 
\paragraph{Tumor growth.} We refer readers to \citep{Bica2020Estimating,melnychuk2022causal} for the complete descriptions of the pharmacokinetic-pharmacodynamic (PK-PD) model. Here we focus on how we introduce distribution shifts by adjusting the treatment bias coefficient $\gamma$.

The volume of tumor after $t$ days of diagnosis is:
\begin{equation}
\begin{aligned}
    &V(t+1) \\
    = &(1 + \rho\log(\frac{K}{V(t)}) - \beta_cC(t) - (\alpha_r d(t) + \beta_r d(t)^2) + e_t)\\
    &\cdot V(t),
\end{aligned}
\end{equation}
where $K, \rho, \beta_c, \alpha_r, \beta_r$ are parameters sampled from the prior distributions defined in \citep{geng2017prediction}. $e_t\sim\mathcal{N}(0,0.01^2)$ is the noise term.

PK-PD model constructs time-varying confounding by connecting the probability of assigning chemotherapy and radiotherapy with the outcome - tumor diameter:
\begin{equation}
\begin{aligned}
    p_c(t) &= \sigma(\frac{\gamma_c}{D_{\text{max}}}(\bar{D}(t) - \delta_c)),\\
    p_r(t) &= \sigma(\frac{\gamma_r}{D_{\text{max}}}(\bar{D}(t) - \delta_r)).
\end{aligned}
\end{equation}
$\bar{D}(t)$ is the mean tumor diameter in the past 15 days and $D_{\text{max}}=13$.~$\sigma$ is the sigmoid function. $\delta_c=\delta_r = D_{\text{max}}/2$. $\gamma_c$ and $\gamma_r$ controls the importance of tumor diameter history on treatment assignment, thus control the strength of time-dependent confounding.

In Tumor growth dataset, we set $\gamma_c=\gamma_r = \gamma = 10$ to generate data in the source domain, and $\gamma_c=\gamma_r=\gamma=0$ for the target domain. As a result, both treatment bias and the data distribution of history differs between source and target domains.

\paragraph{Semi-synthetic MIMIC-III.}
We split the semi-synthetic MIMIC-III dataset introduced in \citep{melnychuk2022causal}~ by ages of patients to the source/target domain. More specifically, we generate simulation data from patients with ages falling in $[20,45]$ as the source domain data and simulation based on patients with ages over 85 as the target domain data.
\rebuttal{Missing values in MIMIC-III dataset is imputed with the so-called “Simple Imputation” described in~\cite{wang2020mimic}. Missing values are first forward filled and then set to individual-specific mean if there are no previous values. If the variable is always missing for a patient, we set it to the global mean.}

\paragraph{M5.}
We adapt the M5 forecasting dataset (https://www.kaggle.com/competitions/m5-forecasting-accuracy) for treatment effect estimation over time. In M5, we select the item pricing as treatment, its sales as outcome and all other features as covariates. We aggregate the item sales by week to reduce the sequence length to the same level as the other two datasets for the convenience of evaluation. We also discretize the continuous pricing by mapping $(p_{t,i} - p_{0,i})/p_{0,i}$ to buckets divided by its 20-quantiles, where $p_{t,i}, p_{0,i}$ are the prices of item $i$ at time $t$ and at its initial sale.

To introduce the feature distribution shift, we select 5000 items in the food category as the source domain data and another 5000 items in the household category as the target domain data.

\section{Baselines}
\paragraph{Baseline implementation.}
We reuse the implementation in \citep{melnychuk2022causal} for evaluating all the baselines, including: MSM~\citep{robins2000marginal}, RMSN~\citep{lim2018forecasting}, CRN~\citep{Bica2020Estimating}, G-Net~\citep{li2021g}, and Causal Transformer (CT)~\citep{melnychuk2022causal}.

\paragraph{Hyperparameter tuning.}
For all baselines, we follow the ranges of hyperparameter tuning in \citep{melnychuk2022causal}~and select the hyperparameters with the lowest factual outcome estimation error on the validation set from the source domain.~\rebuttal{For each method and each dataset, the same set of hyperparameters are used in the zero-shot transfer/data-efficient transfer/standard supervised learning settings. The detailed hyperparameters used for baselines and \modelshortname~are listed in the configuration files in our code repository. Here we list the main hyperparameters for reference.}

\rebuttal{\textbf{MSM}. There is no tuneable hyperparameter in MSM.}

\rebuttal{\textbf{RMSN}. We list the hyperparameters of RMSN in Table~\ref{tab:rmsn-hparams}.}
\begin{table*}[htbp]
\centering
\caption{\rebuttal{RMSN hyperparameters.}}
\label{tab:rmsn-hparams}
\begin{tabular}{@{}llccc@{}}
\toprule
                      &                   & Tumor growth & Semi-synthetic MIMIC-III & M5 \\ \midrule
Propensity Treatment & RNN Hidden Units  & 8             &  6                        & 44   \\
                      & Dropout           &  0.5            & 0.1                         & 0.4   \\
                      & Layer Num         &     1         & 2                         & 1   \\
                      & Max Gradient Norm &      1.0        & 0.5                          & 2.0   \\
                      & Batch Size        &         128     & 256                         & 128   \\
                      & Learning Rate     &            0.01  & 0.01                         & 0.001   \\
\midrule
Propensity History   & RNN Hidden Units  & 24             & 74                          & 92   \\
                      & Dropout           & 0.1             & 0.5                         & 0.5   \\
                      & Layer Num         &  1            & 2                         & 2   \\
                      & Max Gradient Norm & 2.0             & 1.0                         & 0.5   \\
                      & Batch Size        &  128            & 64                         & 128   \\
                      & Learning Rate     &     0.01         & 0.001                         & 0.01  \\
\midrule
Encoder   & RNN Hidden Units  & 24             & 74                         & 46   \\
                      & Dropout           & 0.1             & 0.1                         & 0.1   \\
                      & Layer Num         &  1            & 1                         &  2  \\
                      & Max Gradient Norm &  0.5            & 0.5                          & 0.5   \\
                      & Batch Size        & 64             & 1024                          & 128   \\
                      & Learning Rate     &  0.01            & 0.001                         & 0.0001  \\
\midrule
Decoder   & RNN Hidden Units  & 48             & 196                         & 45   \\
                      & Dropout           & 0.1             & 0.1                         & 0.1   \\
                      & Layer Num         &  1            & 1                         & 1   \\
                      & Max Gradient Norm & 0.5             & 0.5                         & 4.0   \\
                      & Batch Size        &  256            & 1024                         &  256  \\
                      & Learning Rate     & 0.0001             & 0.0001                         & 0.0001  \\
\bottomrule
\end{tabular}
\end{table*}

\rebuttal{\textbf{CRN(ERM)}. See Table \ref{tab:crn-erm-hparams}.}
\begin{table*}[htbp]
\centering
\caption{\rebuttal{CRN(ERM) hyperparameters.}}
\label{tab:crn-erm-hparams}
\begin{tabular}{@{}llccc@{}}
\toprule
        &                               & Tumor growth & Semi-synthetic MIMIC-III & M5 \\ \midrule
Encoder & RNN Hidden Units              & 24             & 74                         & 46   \\
        & Balancing Representation Size & 18             & 74                         & 46   \\
        & FC Hidden Units               & 18             & 37                         & 46   \\
        & Layer Num & 1 & 1 & 2 \\
        & Dropout                       & 0.1             & 0.1                         & 0.1   \\
        & Batch Size                    &  256            & 64                         & 128    \\
        & Learning Rate                 &  0.01            & 0.001                         & 0.001    \\
\midrule
Decoder & RNN Hidden Units              &  18            & 74                         & 46   \\
        & Balancing Representation Size &  6            &  98                        & 90   \\
        & FC Hidden Units               &  6            &  98                        & 22   \\
        & Layer Num & 1 & 2 & 1 \\
        & Dropout                       &  0.1            & 0.1                         & 0.1   \\
        & Batch Size                    &  256            & 256                         & 256   \\
        & Learning Rate                 &  0.001            & 0.0001                         & 0.0001   \\
\bottomrule
\end{tabular}
\end{table*}

\rebuttal{\textbf{CRN}. See Table \ref{tab:crn-hparams}.}
\begin{table*}[htbp]
\centering
\caption{\rebuttal{CRN hyperparameters.}}
\label{tab:crn-hparams}
\begin{tabular}{@{}llccc@{}}
\toprule
        &                               & Tumor growth & Semi-synthetic MIMIC-III & M5 \\ \midrule
Encoder & RNN Hidden Units              & 18             & 74                         & 46   \\
        & Balancing Representation Size & 3             & 74                         & 46   \\
        & FC Hidden Units               & 12             & 37                         & 46   \\
        & Layer Num & 1 & 1 & 2 \\
        & Dropout                       & 0.2             & 0.1                         & 0.1   \\
        & Batch Size                    &  256            & 64                         & 128    \\
        & Learning Rate                 &  0.001            & 0.001                         & 0.001    \\
\midrule
Decoder & RNN Hidden Units              &  3            & 74                         & 46   \\
        & Balancing Representation Size &  3            &  98                        & 90   \\
        & FC Hidden Units               &  3            &  98                        & 22   \\
        & Layer Num & 1 & 2 & 1 \\
        & Dropout                       &  0.2            & 0.1                         & 0.1   \\
        & Batch Size                    &  256            & 256                         & 256   \\
        & Learning Rate                 &  0.001            & 0.0001                         & 0.0001   \\
\bottomrule
\end{tabular}
\end{table*}

\rebuttal{\textbf{CT(ERM)}. See Table \ref{tab:ct-erm-hparams}.}
\begin{table*}[htbp]
\centering
\caption{\rebuttal{CT(ERM) hyperparameters.}}
\label{tab:ct-erm-hparams}
\begin{tabular}{@{}lccc@{}}
\toprule
& Tumor growth & Semi-synthetic MIMIC-III & M5 \\ \midrule
Transformer Hidden Units & 12 & 24 & 24  \\
Balancing Representation Size & 2 & 88 & 94 \\
FC Hidden Units & 12 & 44 & 47 \\
Layer Num & 1 & 1 & 2 \\
Head Num & 2 & 3 & 2 \\
Max Relative Position & 15 & 20 & 30 \\
Dropout & 0.1 & 0.1 & 0.1 \\
Batch Size & 64 & 64 & 64 \\
Learning Rate & 0.001 & 0.01 & 0.001 \\
\bottomrule
\end{tabular}
\end{table*}

\rebuttal{\textbf{CT}. See Table \ref{tab:ct-hparams}.}
\begin{table*}[htbp]
\centering
\caption{\rebuttal{CT hyperparameters.}}
\label{tab:ct-hparams}
\begin{tabular}{@{}lccc@{}}
\toprule
& Tumor growth & Semi-synthetic MIMIC-III & M5 \\ \midrule
Transformer Hidden Units & 16 & 24 & 24  \\
Balancing Representation Size & 16 & 88 & 94 \\
FC Hidden Units & 16 & 44 & 47 \\
Layer Num & 1 & 1 & 2 \\
Head Num & 2 & 3 & 2 \\
Max Relative Position & 15 & 20 & 30 \\
Dropout & 0.2 & 0.1 & 0.1 \\
Batch Size & 64 & 64 & 64 \\
Learning Rate & 0.001 & 0.01 & 0.001 \\
\bottomrule
\end{tabular}
\end{table*}

\rebuttal{\textbf{G-Net}. See Table \ref{tab:gnet-hparams}.}
\begin{table*}[htbp]
\centering
\caption{\rebuttal{G-Net hyperparameters.}}
\label{tab:gnet-hparams}
\begin{tabular}{@{}lccc@{}}
\toprule
& Tumor growth & Semi-synthetic MIMIC-III & M5 \\ \midrule
RNN Hidden Units & 24 & 148 & 144 \\
FC Hidden Units & 48 & 74 & 72 \\
Dropout & 0.1 & 0.1 & 0.1 \\
Layer Num & 1 & 1 & 2 \\
Batch Size & 128 & 256 & 256 \\
Learning Rate & 0.001 & 0.01 & 0.001 \\
\bottomrule
\end{tabular}
\end{table*}

\rebuttal{\textbf{\modelshortname}. See Table \ref{tab:ours-hparams}.}
\begin{table*}[htbp]
\centering
\caption{\rebuttal{\modelshortname~hyperparameters.}}
\label{tab:ours-hparams}
\begin{tabular}{@{}llccc@{}}
\toprule
        &                               & Tumor growth & Semi-synthetic MIMIC-III & M5 \\ \midrule
Encoder & Transformer Hidden Units              & 24             & 36                         & 36   \\
        & Encoder Momentum & 0.99             & 0.99                         & 0.99   \\
        & Temperature               & 1.0             & 1.0                         & 1.0   \\
        & Layer Num & 1 & 1 & 2 \\
        & Head Num & 2 & 3 & 2 \\
        & Dropout                       & 0.1             & 0.1                         & 0.1   \\
        & Batch Size                    &  64            & 64                         & 64    \\
        & Learning Rate                 &  0.001            & 0.001                         & 0.001    \\
\midrule
Decoder & Hidden Units              &  128            & 128                         & 128   \\
        & Batch Size                    &  32            & 32                         & 32   \\
        & Learning Rate                 &  0.001            & 0.001                         & 0.001   \\
\bottomrule
\end{tabular}
\end{table*}

\paragraph{Comparison of numbers of model parameters.}
Here we list the number of trainable parameters in each baseline as well as \modelshortname~in the experiments of each dataset.
\begin{table*}[htbp]
\centering
\caption{\rebuttal{Number of trainable parameters.}}
% \resizebox{\linewidth}{!}{%
\begin{tabular}{@{}llll@{}}
\toprule
\#trainable params          & Tumor growth & semi-synthetic MIMIC-III & M5    \\ \midrule
MSM & $<$1K & (-) & (-) \\
RMSN                        & 18.8K                   & 387K     & 213K  \\
CRN(ERM)      & 6.5K                    & 164K     & 78K   \\
CRN                         & 2.3K                    & 164K     & 78K   \\
CT(ERM)        & 5.2K                    & 45K      & 80.3K \\
CT                          & 9.4K                    & 45K      & 80.3K \\
G-Net & 3.4K & 151K & 323K\\
\modelshortname         & 20.7K                   & 43.6K    & 77.5K\\
\bottomrule
\end{tabular}%
% }
\end{table*}

\section{Results of supervised learning setup}
Table~\ref{tab:noncold-merged-results} shows the performance in standard supervised learning setting, with both train and test data from the source domain. Overall, \modelshortname~outperforms other baselines in tumor growth and semi-synthetic MIMIC-III datasets. With M5, \modelshortname~also shows comparable performance to the CT(ERM) with a 1.3\% relative difference.
% For longer estimation horizons ($\tau\geq 3)$, \modelshortname~outperforms all baselines and shows an average improvement of \todo{add percentage} over the second-best performing baselines.
% For shorter horizons ($\tau=1$ and $\tau=2$), \modelshortname~shows inferior performance instead.

\begin{table*}[!h]
\centering
\caption{Results in standard supervised learning setting, with source and target datasets coming from the same distribution for multi-step outcome estimation. We report the mean +- standard deviation of Rooted Mean Squared Errors (RMSEs) over 5 runs. \textbf{Bold}: the best results. \underline{Underline}: the 2nd best results.}
\label{tab:noncold-merged-results}
\resizebox{\textwidth}{!}{

\begin{tabular}{@{}l|l|cccccc|c|c@{}}
\toprule
Dataset & Method & $\tau=1$ & $\tau=2$ & $\tau=3$ & $\tau=4$ & $\tau=5$ & $\tau=6$ & Avg & Gain(\%) \\ \midrule
\multirow{8}{*}{\begin{tabular}[c]{@{}l@{}}Tumor\\ growth\end{tabular}}  & MSM & \meanstd{5.8368}{0.6157}& \textbf{\meanstd{2.0400}{0.6719}}&	\textbf{\meanstd{3.0385}{0.9990}}&	\underline{\meanstd{3.8701}{1.2736}}&	\underline{\meanstd{4.6173}{1.5246}}&	\meanstd{5.3823}{1.7839} & \underline{\meanstd{4.1308}{1.1211}} & 12.3\%\\
% & RMSN & \meanstd{4.9266}{0.5383}&	\meanstd{7.2953}{2.7098}&	\meanstd{8.5150}{3.1892}&	\meanstd{8.9795}{3.3301}&	\meanstd{8.9954}{3.2845}&	\meanstd{8.6705}{3.1093}\\
& RMSN & \meanstd{4.8388}{0.7770}&	\meanstd{5.4447}{1.9202}&	\meanstd{5.9261}{2.1096}&	\meanstd{5.9817}{2.1270}&	\meanstd{5.8705}{2.0544}&	\meanstd{5.5461}{1.8865} & \meanstd{5.6013}{1.7727}& 35.4\%\\
& CRN(ERM)&	\meanstd{5.1601}{0.5222}&	\meanstd{6.0784}{2.3196}&	\meanstd{6.4721}{2.4221}&	\meanstd{6.6142}{2.4206}&	\meanstd{6.5648}{2.3455}&	\meanstd{6.2939}{2.1955}& \meanstd{6.1972}{2.0226}& 41.6\%\\
% & CRN&	\meanstd{4.9781}{0.3169}&	\meanstd{11.5230}{9.9577}&	\meanstd{15.2778}{15.9393}&	\meanstd{17.3213}{19.6976}&	\meanstd{18.5698}{22.0477}&	\meanstd{19.2921}{23.5647}& \meanstd{14.4937}{15.2290}& 75.0\%\\
& CRN& \meanstd{4.8130}{0.2296}	&\meanstd{6.3126}{2.9523}	& \meanstd{6.6993}{3.8805}	&\meanstd{6.7520}{3.8551}&	\meanstd{6.8386}{3.5630}	&\meanstd{6.8852}{3.1150}	&\meanstd{6.3834}{2.8863}	& 43.3\%\\
& CT(ERM)&	\meanstd{5.1286}{1.3377}&	\meanstd{5.7262}{2.7601}&	\meanstd{6.5085}{2.9886}&	\meanstd{6.9248}{3.0009}&	\meanstd{7.1971}{2.9346}&	\meanstd{7.2369}{2.7570}& \meanstd{6.4537}{2.5904}& 43.9\%\\
& CT&	\meanstd{6.5485}{1.5221}&	\meanstd{7.5382}{2.8528}&	\meanstd{7.9030}{2.9569}&	\meanstd{7.9828}{2.9332}&	\meanstd{7.8244}{2.8075}&	\meanstd{7.4418}{2.6103}& \meanstd{7.5398}{2.5976}& 52.0\%\\
% & G-Net & \textbf{\meanstd{3.8927}{0.3665}} & \underline{\meanstd{4.0104}{1.3071}} & \meanstd{4.9584}{1.5950} & \meanstd{5.3751}{1.7249} & \meanstd{5.4647}{1.7300} & \underline{\meanstd{5.2775}{1.6330}} \\
& G-Net & \meanstd{3.9371}{0.4023}&	\meanstd{3.7697}{1.1861}&	\meanstd{4.6054}{1.4181}&	\meanstd{4.9730}{1.4773}&	\meanstd{5.0491}{1.4410}&	\underline{\meanstd{4.8745}{1.3153}} & \meanstd{4.5348}{1.1778} & 20.1\%\\
\cmidrule(l){2-10}
% & \modelshortname & \underline{\meanstd{4.7767}{1.5867}}&	\meanstd{4.0354}{2.2743}&	\underline{\meanstd{4.5238}{2.0499}}&	\underline{\meanstd{4.9846}{2.1040}}&	\underline{\meanstd{5.2103}{1.9996}}&	\textbf{\meanstd{5.1183}{1.7762}}\\
% & \modelshortname & \underline{\meanstd{4.1150}{0.6488}}&	\underline{\meanstd{3.4111}{0.9540}}&	\underline{\meanstd{3.6497}{1.1376}}&	\textbf{\meanstd{3.6926}{1.1114}}&	\textbf{\meanstd{3.7122}{1.0727}}&	\textbf{\meanstd{3.6402}{1.0041}}\\
& \modelshortname & \textbf{\meanstd{3.7403}{0.3695}}&	\underline{\meanstd{3.0067}{0.9065}}&	\underline{\meanstd{3.4619}{1.1557}}	&\textbf{\meanstd{3.8501}{1.3127}}	&\textbf{\meanstd{3.9160}{1.3142}}	&\textbf{\meanstd{3.7525}{1.1493}} & \textbf{\meanstd{3.6212}{1.0040}} & (-)\\
\midrule
\multirow{8}{*}{\begin{tabular}[c]{@{}l@{}}Semi-\\synthetic\\MIMIC-III\end{tabular}} 
% & MSM & \meanstd{0.4328}{0.0831} & \meanstd{0.5843}{0.1638} & $>10$ & $>10$ & $>10$ & $>10$\\
& RMSN&	\underline{\meanstd{0.2107}{0.0261}}&	\meanstd{0.5352}{0.0842}&	\meanstd{0.6722}{0.1096}&	\meanstd{0.7669}{0.1203}&	\meanstd{0.8309}{0.1280}&	\meanstd{0.8764}{0.1331}& \meanstd{0.6487}{0.0976}& 21.7\%\\
& CRN(ERM)&	\textbf{\meanstd{0.1951}{0.0202}}&	\meanstd{0.4426}{0.0799}&	\meanstd{0.5530}{0.0859}&	\underline{\meanstd{0.6113}{0.0842}}&	\underline{\meanstd{0.6478}{0.0828}}&	\underline{\meanstd{0.6708}{0.0819}} & \underline{\meanstd{0.5201}{0.0713}} & 2.4\%\\
& CRN&	\meanstd{0.3276}{0.0301}&	\meanstd{0.5234}{0.0839}&	\meanstd{0.6531}{0.0985}&	\meanstd{0.7234}{0.0985}&	\meanstd{0.7618}{0.0921}&	\meanstd{0.7825}{0.0854}& \meanstd{0.6286}{0.0801}& 19.2\%\\
& CT(ERM)&	\meanstd{0.2130}{0.0164}&	\meanstd{0.4426}{0.0766}&	\meanstd{0.5495}{0.0836}&	\meanstd{0.6191}{0.0851}&	\meanstd{0.6669}{0.0856}&	\meanstd{0.7010}{0.0834}& \meanstd{0.5320}{0.0695}& 4.6\%\\
& CT&	\meanstd{0.2175}{0.0178}&	\underline{\meanstd{0.4421}{0.0757}}&	\underline{\meanstd{0.5458}{0.0854}}&	\meanstd{0.6161}{0.0925}&	\meanstd{0.6670}{0.0993}&	\meanstd{0.7047}{0.1040}& \meanstd{0.5322}{0.0765}& 4.6\%\\
& G-Net & \meanstd{0.3418}{0.0290} & \meanstd{0.6015}{0.0653} & \meanstd{0.7542}{0.0758} & \meanstd{0.8620}{0.0825} & \meanstd{0.9429}{0.0875} & \meanstd{1.0035}{0.0915}& \meanstd{0.7510}{0.0686}& 32.4\%\\
% \meanstd{1.5956}{0.4332} & \meanstd{1.6938}{0.4584} & \meanstd{1.7776}{0.4788} & \meanstd{1.8492}{0.4912} & \meanstd{1.9127}{0.4996}
\cmidrule(l){2-10}
% & \modelshortname & \meanstd{0.2763}{0.0283}&	\meanstd{0.4504}{0.0922}&	\textbf{\meanstd{0.5327}{0.1042}}&	\textbf{\meanstd{0.5811}{0.1072}}&	\textbf{\meanstd{0.6114}{0.1096}}&	\textbf{\meanstd{0.6320}{0.1106}}\\
& \modelshortname & \meanstd{0.2286}{0.0265}&	\textbf{\meanstd{0.4417}{0.0876}}&	\textbf{\meanstd{0.5288}{0.0957}}&	\textbf{\meanstd{0.5825}{0.0991}}&	\textbf{\meanstd{0.6190}{0.1018}}&	\textbf{\meanstd{0.6458}{0.1030}}& \textbf{\meanstd{0.5077}{0.0848}}& (-) \\\midrule
\multirow{8}{*}{M5}
% & MSM & $>50$ &$>50$ &$>50$ &$>50$ &$>50$ &$>50$   \\
& RMSN&	\meanstd{35.7795}{4.3603}&	\meanstd{33.2570}{2.3870}&	\meanstd{33.4138}{4.0678}	&\meanstd{33.4169}{4.4289}&	\meanstd{33.3104}{4.4017}&	\meanstd{33.3819}{4.1602}& \meanstd{33.7599}{3.9379}& 52.2\%\\
& CRN(ERM) &	\meanstd{13.8445}{0.1550}	&\meanstd{15.7926}{0.1278}&	\meanstd{16.7071}{0.2240}&	\meanstd{17.0887}{0.1724}&	\meanstd{17.2709}{0.0923}&	\meanstd{17.9759}{0.0880}& \meanstd{16.4466}{0.1367} & 1.8\%\\
& CRN&	\underline{\meanstd{13.5907}{0.0859}}&	{\meanstd{15.5242}{0.0692}}&	\meanstd{16.2694}{0.1157}&	\underline{\meanstd{16.7355}{0.0719}}&	\underline{\meanstd{17.0095}{0.0388}}&	\underline{\meanstd{17.6874}{0.0558}}& \underline{\meanstd{16.1361}{0.0673}}& -0.1\%\\
& CT(ERM) &	\textbf{\meanstd{13.4887}{0.1335}}&	\underline{\meanstd{15.3397}{0.2922}}	&\meanstd{16.3415}{0.4096}&	\meanstd{17.0545}{0.5603}&	\meanstd{17.4828}{0.5609}&
\meanstd{18.5832}{0.4930}& \textbf{\meanstd{15.9414}{0.3804}}&-1.3\%\\
& CT&	\meanstd{13.6721}{0.3574}&	\meanstd{15.9384}{1.0910}&	\meanstd{17.2781}{1.6049}&	\meanstd{18.1796}{1.8134}&	\meanstd{18.8805}{2.2159}&
\meanstd{19.2510}{1.9642}& \meanstd{16.7897}{1.4144}& 3.8\%\\
% & G-Net & \meanstd{14.0946}{0.1622} & \meanstd{18.9542}{0.4370} & \meanstd{21.1086}{0.6713} & \meanstd{22.8285}{0.7842} & \meanstd{24.2457}{0.9500} & \meanstd{25.4337}{1.0244}  \\
& G-Net & \meanstd{13.7187}{0.0833}& \textbf{\meanstd{14.9851}{0.1205}}& \textbf{\meanstd{15.9578}{0.1701}}& \meanstd{16.8278}{0.2229}& \meanstd{17.4833}{0.3111}& \meanstd{18.1665}{0.3795}& \meanstd{16.1898}{0.2070}& 0.3\% \\
\cmidrule(l){2-10}
& \modelshortname & \meanstd{14.2556}{0.1792}&	\meanstd{15.6151}{0.2287}&	\underline{\meanstd{16.1743}{0.2076}}&	\textbf{\meanstd{16.4791}{0.1276}}&	\textbf{\meanstd{16.9037}{0.1322}}&	\textbf{\meanstd{17.4379}{0.1845}}& \meanstd{16.1443}{0.1742}& (-) \\
\bottomrule
\end{tabular}
}
\end{table*}

\section{Visualization of the Learned Representations}
% \begin{wrapfigure}{r}{0.38\textwidth}
\begin{figure}[htbp]
    \centering
    \includegraphics[width=\linewidth]{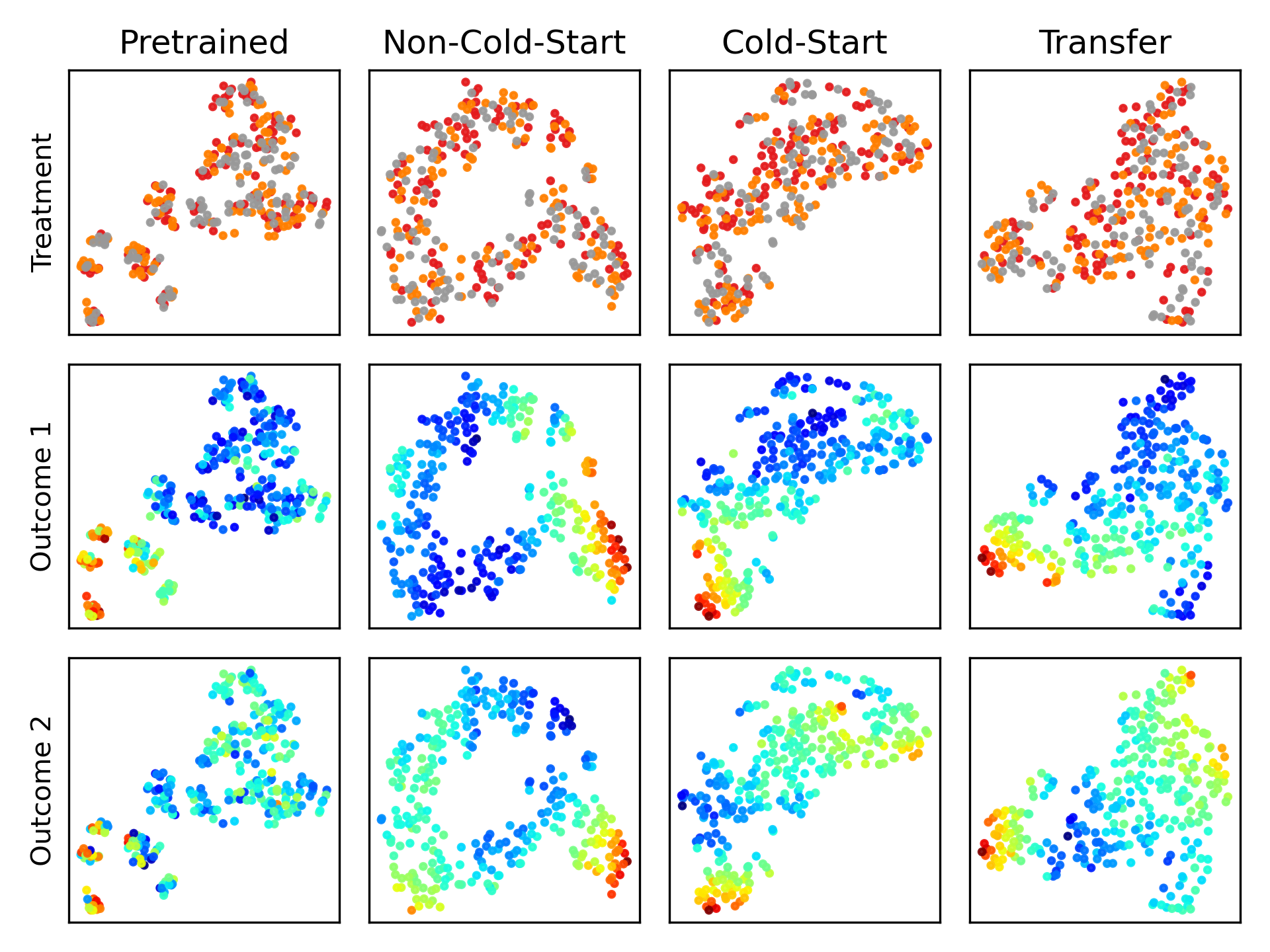}
    \caption{T-SNE visualization of learned representations in Semi-synthetic MIMIC-III dataset.
    % , which shows that \modelshortname~learns representations that are (1) balanced w.r.t treatments, (2) predictive of outcomes, and (3) transferable from the source domain to the target domain.
    % The column name indicates the stage after which the representations are produced. The row name shows the variable used for coloring.
    }
    \label{fig:vis-rep}
% \end{wrapfigure}
\end{figure}
Fig.~\ref{fig:vis-rep} depicts the representations learned for the Semi-synthetic MIMIC-III dataset after each of the 4 stages: (a) \textbf{Pretrained}: representations after the self-supervised learning stage of source data. (b) \textbf{Non-Cold-Start}: representations fine-tuned with factual outcome estimation loss of source data. (c) \textbf{Cold-Start}: representations of target data when directly applying the encoder trained in (b). 
(d) \textbf{Transfer}: representations of target data after fine-tuned with small amount of target data. We use T-SNE to map each representation to a 2D space and color each point with values of its upcoming treatment and outcomes.

As shown in the first row, representations with different types of upcoming treatments overlap, indicating that the learned representations after each stage are balanced towards treatments.
In the second and the third rows, we observe clusters of representations corresponding to similar outcome values, which indicates that the learned representations are informative about the upcoming outcomes, even including the representations trained only with self-supervised loss (column ``Pretrained"). Such clustered structures also persist when moving from the source domain data (column ``Non-Cold-Start") to the target domain data (``Cold-Start"), showing that the learned representations can generalize to cold-start cases.

\section{Examples of counterfactual treatment outcome estimation}
\begin{figure*}[!h]
    \centering
    \includegraphics[width=\textwidth]{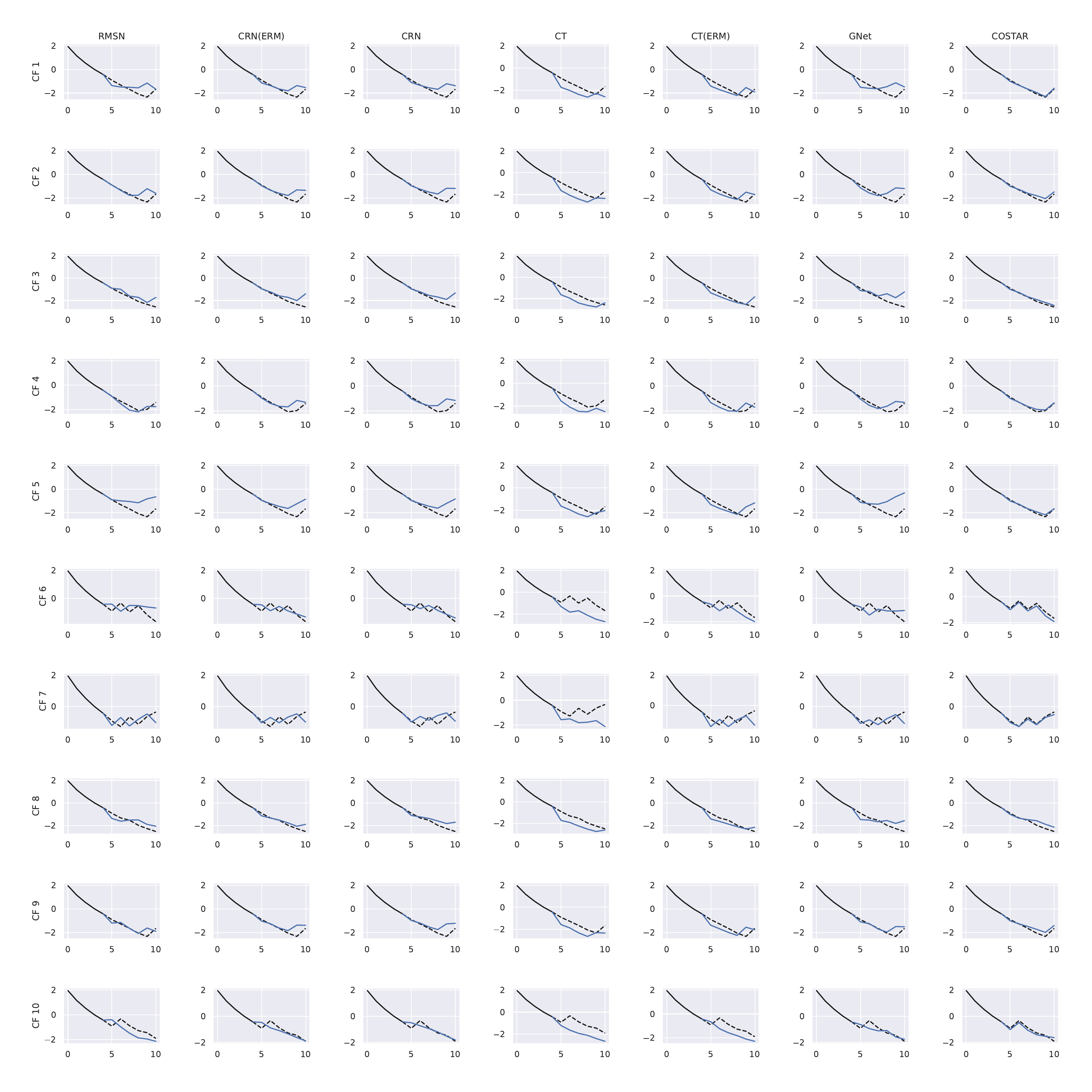}
    \caption{Examples of counterfactual treatment outcome estimation with semi-synthetic MIMIC-III data in the zero-shot transfer setting. We plot one of the two output dimensions for clarity. Each row lists the results of a counterfactual treatment sequence, while each column shows the estimations of one method across all treatment sequences tested. In each sub-figure, the observed historical outcomes are plotted in black solid lines, and the ground truth counterfactual outcomes in black dash lines. The blue solid lines show the estimated outcomes. }
    \label{fig:example-cf}
\end{figure*}

Fig.~\ref{fig:example-cf}~qualitatively compare the counterfactual outcome estimation performance differences between \modelshortname~and baselines in the zero-shot transfer setting. We randomly select a sequence from the observed data until time $t=4$ (x-axis), then apply sequences of treatments sampled uniformly (i.e. no treatment bias) and simulate the step-wise outcomes for 10 times as the ground truth. We compare the ground truth of each simulation with all methods tested with semi-synthetic MIMIC-III dataset. In Fig.~\ref{fig:example-cf}~ we find that the gaps between estimations and ground truth outcomes are obvious in columns of baseline results. Instead, they closely match each other in the estimation results (the rightmost column) given by \modelshortname, demonstrating its superior performance.

\end{document}